\documentclass[11pt,letterpaper]{article}

\usepackage{paralist,amsmath,amssymb,fullpage}
\usepackage[ruled,vlined]{algorithm2e}

\title{Learning implicitly in reasoning in PAC-Semantics}
\date{}
\author{Brendan Juba\thanks{Supported by ONR grant number N000141210358 and NSF Grant CCF-0939370.}\\ MIT CSAIL and Harvard SEAS\\ {\tt bjuba@alum.mit.edu}}

\newtheorem{theorem}{Theorem}

\newtheorem{definition}[theorem]{Definition}

\newtheorem{proposition}[theorem]{Proposition}
\newtheorem{corollary}[theorem]{Corollary}

\def\FullBox{\hbox{\vrule width 8pt height 8pt depth 0pt}}

\newcommand{\qed}{\;\;\;\FullBox}

\newenvironment{proof}{\noindent{\bf Proof:~~}}{\qed}

\newcommand{\calS}{\mathcal{S}}
\newcommand{\bbN}{\mathbb{N}}
\newcommand{\bbF}{\mathbb{F}}
\newcommand{\bbQ}{\mathbb{Q}}
\newcommand{\width}{\mathsf{width}}
\newcommand{\poly}{\mathrm{poly}}
\newcommand{\supp}{\mathrm{supp}}

\begin{document}
\maketitle

\begin{abstract}
We consider the problem of answering queries about formulas of propositional
logic based on background knowledge partially represented explicitly as other 
formulas, and partially represented as {\em partially obscured examples} 
independently drawn from a fixed probability distribution, where the queries are
answered with respect to a weaker semantics than usual -- PAC-Semantics, 
introduced by Valiant~\cite{valiant00} -- that is defined using the distribution
of examples. We describe a fairly general, efficient reduction to limited 
versions of the decision problem for a proof system (e.g., bounded space 
treelike resolution, bounded degree polynomial calculus, etc.) from 
corresponding versions of the reasoning problem where some of the background 
knowledge is not explicitly given as formulas, only learnable from the examples.
Crucially, we do {\em not} generate an explicit representation of the 
knowledge extracted from the examples, and so the ``learning'' of the background
knowledge is only done {\em implicitly}. As a consequence, this approach can 
utilize formulas as background knowledge that are {\em not} perfectly valid over
the distribution---essentially the analogue of agnostic learning here.
\end{abstract}

\newpage

\section{Introduction}

PAC-Semantics was introduced by Valiant~\cite{valiant00} in an attempt to unify
statistical and logical approaches to reasoning: on the one hand, given 
background knowledge represented as a collection of axioms, one may perform
logical deduction, and on the other hand, given background knowledge represented
as a collection of examples, one can derive a statistical conclusion by testing
whether the conclusion is supported by a sufficiently large fraction of the
examples. PAC-Semantics captures {\em both} sources. 
% Motivation: story about the aviary
% Observed many birds in the aviary that will not eat fish.
% Knowledge: penguins eat fish; birds fly unless they are penguins.
% From this we can conclude that the birds in the aviary fly.
As is typical for such works, we can illustrate the utility of such a combined
approach with a story about an aviary. Suppose that we know that the birds of 
the aviary fly unless they are penguins, and that penguins eat fish. Now, 
suppose that we visit the aviary at feeding time, and notice that most (but 
perhaps not all) of the birds in the aviary seem not to eat fish. From this 
information, we can infer that most of the birds in the aviary can fly. This
conclusion draws on both the empirical (partial) information and reasoning from 
our explicit, factual knowledge: on the one hand, our empirical observations did
not mention anything about whether or not the birds of the aviary could fly, and
on the other hand, although our knowledge is sufficient to conclude that the 
birds that don't eat fish can fly, it isn't sufficient to conclude whether or 
not, broadly speaking, the birds in the aviary can fly.
% Story combines empirical, partial information and reasoning from 
% factual knowledge---the empirical samples reveal nothing about whether or
% not the birds of the aviary fly on the one hand, and on the other hand,
% the stated knowledge alone is insufficient to conclude whether or not
% the birds in the aviary fly.

Valiant's original work described an application of PAC-Semantics to the task of
predicting the values of unknown attributes in new examples based on the values 
of some known attributes of the example---for example, filling in a missing word
in an example sentence~\cite{mv08}. In this work, by contrast, we introduce and 
describe how to solve a (limited) {\em decision} task for PAC-Semantics, 
deciding whether or not a given ``query'' formula follows from the background 
knowledge, represented by {\em both} a collection of axiom formulas and a 
collection of examples. In particular, we use a model of partial information due
to Michael~\cite{michael10} to capture and cope with reasoning from 
{\em partially obscured} examples from a target distribution.

What we show is roughly that as long as we can efficiently use small proofs to 
certify validity in the classical sense and the rules of inference in the proof 
system are preserved under restrictions, we can efficiently certify the validity
(under PAC-Semantics) of a query from a sample of partial assignments whenever 
it follows from some formula(s) that could be verified to hold under the partial
assignments. Thus, in such a case, the introduction of probability to the 
semantics in this limited way (to cope with the imperfection of learned rules) 
actually does not harm the tractability of inference. Moreover, the ``learning''
is actually also quite efficient, and imposes {\em no} restrictions on the 
representation class beyond the assumption that their values are observed under 
the partial assignments and the restrictions imposed by the proof system itself.
In Section~\ref{proof-systems}, we will then observe that almost every special 
case of a propositional proof system with an efficient decision algorithm
considered in the literature satisfies these conditions, establishing the 
breadth of applicability of the approach.

It is perhaps more remarkable in from a learning theoretic perspective that our
approach does not require the rules to be learned (or discovered) to be 
completely consistent with the examples drawn from the (arbitrary) distribution.
In the usual learning context, this would be referred to as {\em agnostic} 
learning, as introduced by Kearns et al.~\cite{kss94}. Agnostic learning is 
notoriously hard---Kearns et al. noted that agnostic learning of conjunctions 
(over an arbitrary distribution, in the standard PAC-learning sense) would yield
an efficient algorithm for PAC-learning DNF (also over arbitrary distributions),
which remains the central open problem of computational learning theory. Again,
by declining to produce a hypothesis, we manage to circumvent a barrier (to the 
state of the art, at least). Such rules of less-than-perfect validity seem to 
be very useful from the perspective of AI: for example, logical encodings 
of planning problems typically use ``frame axioms'' that assert that nothing 
changes unless it is the effect of an action. In a real world setting, these 
axioms are not strictly true, but such rules still provide a useful 
approximation. It is therefore desirable that we can learn to utilize them.
We discuss this further in Section~\ref{imperfect-rules}.

\paragraph{Relationship to other work}
Given that the task we consider is fundamental and has a variety of 
applications, other approaches have naturally been proposed---for example,
Markov Logic~\cite{rd06} is one well-known approach based on graphical 
models, and Bayesian Logic Programming~\cite{kdr08} is an approach that has 
grown out of the Inductive Logic Programming (ILP) community that can address 
the kinds of tasks we consider here. The main distinction between all of these 
approaches and our approach is that these other approaches {\em all aim to model
the distribution of the data,} which is generally a {\em much} more demanding 
task -- both in terms of the amount of data and computation time required -- 
than simply answering a query. Naturally, the upshot of these other works is 
that they are much more versatile, and there are a variety of other tasks (e.g.,
density estimation, maximum likelihood computations) that these frameworks can 
handle that we do not. Our aim is instead to show how this more limited (but 
still useful) task can be done much more efficiently, much like how algorithms
such as SVMs and boosting can succeed at predicting attributes 
without needing to model the distribution of the data. 

In this respect, our work is similar to the Learning to Reason framework of 
Khardon and Roth~\cite{kr97}, who showed how an NP-hard reasoning task (deciding
a $\log n$-CNF query), when coupled with a learning task beyond the reach of the
state of the art (learning DNF from random examples) could result in an 
efficient overall system. The distinction between our work and Khardon and 
Roth's is, broadly speaking, that we re-introduce the theorem-proving aspect 
that Khardon and Roth had explicitly sought to avoid. Briefly, these techniques 
permit us to incorporate declaratively specified background knowledge and 
moreover, permit us to cope with partial information in more general cases than 
Khardon and Roth~\cite{kr99}, who could only handle constant width clauses. 
Another difference between our work and that of Khardon and Roth, that also 
distinguishes our work from traditional ILP (e.g., \cite{mdr94}), is that as
mentioned above, we are able to utilize rules that hold with less than perfect 
probability (akin to agnostic learning, but easier to achieve here). 

\section{Definitions and preliminaries}

\paragraph{PAC-Semantics}%\label{psdef}
% PAC Semantics -- propositional logic and distribution
Inductive generalization (as opposed to deduction) inherently entails the 
possibility of making mistakes. Thus, the kind of rules produced by learning
algorithms cannot hope to be valid in the traditional (Tarskian) sense (for
reasons we describe momentarily), but intuitively they do capture some useful
quality. PAC-Semantics were thus introduced by Valiant~\cite{valiant00} to 
capture the quality possessed by the output of PAC-learning algorithms when 
formulated in a logic. Precisely, suppose that we observe examples independently
drawn from a distribution over $\{0,1\}^n$; now, suppose
that our algorithm has found a rule $f(x)$ for predicting some target attribute 
$x_t$ from the other attributes. The formula ``$x_t=f(x)$'' may not be valid in 
the traditional sense, as PAC-learning does not guarantee that the rule
holds for every possible binding, only that the rule $f$ so
produced agrees with $x_t$ with probability $1-\epsilon$ with respect to future 
examples drawn from the same distribution. That is, the formula {\em is} instead
``valid'' in the following sense:

% (1-eps)-valid formulae

\begin{definition}[$(1-\epsilon)$-valid]
Given a distribution $D$ over $\{0,1\}^n$, we say that a Boolean function
$R$ is {\em $(1-\epsilon)$-valid} if $\Pr_{x\in D}[R(x)=1]\geq 1-\epsilon$.
If $\epsilon=0$, we say $R$ is {\em perfectly valid}.
\end{definition}

Of course, we may consider $(1-\epsilon)$-validity of relations $R$ that are not
obtained by learning algorithms and in particular, not of the form 
``$x_t=f(x)$.''

\paragraph{Classical inference in PAC-Semantics.}
Valiant~\cite{valiant00} considered one rule of inference, {\em chaining}, for 
formulas of the form $\ell_t=f(x)$ where $f$ is a linear threshold function: 
given a collection of literals such that the partial assignment obtained from
satisfying those literals guarantees $f$ evaluates to true, infer the literal
$\ell_t$. Valiant observed that for such learned formulas, the conjunction of 
literals derived from a sequence of applications of chaining is also $1-
\epsilon'$-valid for some polynomially larger $\epsilon'$.
It turns out that this property of soundness under PAC-Semantics is not a
special feature of chaining: generally, it follows from the union bound that any
classically sound derivation is also sound under PAC-Semantics in a similar 
sense.
\begin{proposition}[Classical reasoning is usable in PAC-Semantics]
\label{classical-inf-bound}
Let $\psi_1,\ldots,\psi_k$ be formulas such that each $\psi_i$ is
$(1-\epsilon_i)$-valid under a common distribution $D$ for some $\epsilon_i\in
[0,1]$. Suppose that $\{\psi_1,\ldots,\psi_k\}\models\varphi$ (in the classical
sense). Then $\varphi$ is $1-\epsilon'$-valid under $D$ for
$\epsilon'=\sum_i\epsilon_i$.
\end{proposition}

So, soundness under PAC-Semantics does not pose any constraints on the rules of
inference that we might consider; the degree of validity of the conclusions 
merely aggregates any imperfections in the various individual premises involved.
We also note that without further knowledge of $D$, the loss 
of validity from the use of a union bound is optimal.

\begin{proposition}[Optimality of the union bound for classical reasoning]
Let $\psi_1,\ldots,\psi_k$ be a collection of formulas such that there
exists some distribution $D$ on which each $\psi_i$ is $1-\epsilon_i$-valid,
for which $\{\psi_1,\ldots,\psi_{i-1},\psi_{i+1},\ldots,\psi_k\}\not\models
\psi_i$, and $\sum_i\epsilon_i<1$. Then there exists a distribution $D'$ for
which each $\psi_i$ is $1-\epsilon_i$-valid, but $\psi_1\wedge\cdots\wedge
\psi_k$ is not $1-\sum_i\epsilon_i+\delta$ valid for any $\delta>0$.
\end{proposition}
\begin{proof}
Since Proposition~\ref{classical-inf-bound} guarantees that $\psi_1\wedge\cdots
\wedge\psi_k$ is at least $1-\sum_i\epsilon_i$-valid where $1-\sum_i\epsilon_i>
0$, there must be a (satisfying) assignment $x^{(0)}$ for $\psi_1\wedge\cdots
\wedge\psi_k$. On the other hand, as each $\psi_i$ is not entailed by the 
others, there must be some assignment $x^{(i)}$ that satisfies the others but 
falsifies $\psi_i$. We now construct $D'$: it places weight $\epsilon_i$ on the 
assignment $x^{(i)}$, and weight $1-\sum_i\epsilon_i$ on $x^{(0)}$. It is easy 
to verify that $D'$ satisfies the claimed conditions.
\end{proof}

Subsequently, we will assume that our Boolean functions will be given by 
formulas of propositional logic formed over Boolean variables $\{x_1,\ldots,
x_n\}$ by negation and the following linear threshold connectives (which we will
refer to as the {\em threshold basis} for propositional formulas):

\begin{definition}[Threshold connective]
A {\em threshold connective} for a list of $k$ formulas $\phi_1,\ldots,\phi_k$
is given by a list of $k+1$ real numbers, $c_1,\ldots,c_k, b$.
The formula $[\sum_{i=1}^kc_i\phi_i\geq b]$ is interpreted as follows:
given a Boolean interpretation for the $k$ formulas, the connective is true if
$\sum_{i:\phi_i=1}c_i\geq b$.
\end{definition}

Naturally, a threshold connective expresses a $k$-ary AND connective by
taking the $c_i=1$, and $b=k$, and expresses a $k$-ary OR by taking
$c_1,\ldots,c_k, b=1$.

We note that Valiant actually defines PAC-Semantics for first-order logic by 
considering $D$ to be a distribution over the values of atomic formulas. He
focuses on formulas of bounded arity over a polynomial size domain; then
evaluating such formulas from the (polynomial size) list of values of all atomic
formulas is tractable, and in such a case everything we consider here about 
propositional logic essentially carries over in the usual way, by considering 
each atomic formula to be a propositional variable (and rewriting the 
quantifiers as disjunctions or conjunctions over all bindings). As we don't 
have any insights particular to first-order logic to offer, we will focus 
exclusively on the propositional case in this work.

\paragraph{Partial observability}
% masking processes and obscured examples
Our knowledge of a distribution $D$ will be provided in the form of a collection
of {\em examples} independently drawn from $D$, and our main question of
interest will be deciding whether or not a formula is $(1-\epsilon)$-valid. Of 
course, reasoning in PAC-Semantics from (complete) examples is trivial:
Hoeffding's inequality guarantees that with high probability, the proportion of
times that the query formula evaluates to {\em `true'} is a good estimate of
the degree of validity of the formula. By contrast, if the distribution $D$ is
not known, then we can't guarantee that a formula is $(1-\epsilon)$-valid for
any $\epsilon<1$ without examples without deciding whether the query is a
tautology. So, it is only interesting to consider what happens ``in between.''
To capture such ``in between'' situations, we will build on the theory of 
learning from partial observations developed by Michael~\cite{michael10}.

\begin{definition}[Partial assignments]
A {\em partial assignment} $\rho$ is an element of $\{0,1,*\}^n$. We say
that a partial assignment $\rho$ is {\em consistent} with an assignment
$x\in\{0,1\}^n$ if whenever $\rho_i\neq *$, $\rho_i=x_i$.
\end{definition}

Naturally, instead of examples from $D$, our knowledge of $D$ will be provided
in the form of a collection of example {\em partial assignments} drawn from a
{\em masking process} over $D$:

\begin{definition}[Masking process]\label{def-mask}
A {\em mask} is a function $m:\{0,1\}^n\to\{0,1,*\}^n$, with the property that 
for any $x\in\{0,1\}^n$, $m(x)$ is consistent with $x$.  A {\em masking
process} $M$ is a mask-valued random variable (i.e., a random function).
We denote the distribution over partial assignments obtained by applying a
masking process $M$ to a distribution $D$ over assignments by $M(D)$.
\end{definition}

Note that the definition of masking processes allows the hiding of entries to
depend on the underlying example from $D$. Of course, since we know that when
{\em all} entries are hidden by a masking process the problem we consider will
become NP-hard, we must restrict our attention to settings where it is possible
to learn something about $D$. In pursuit of this, we will consider formulas that
can be evaluated in the straightforward way from the partial assignments with
high probability---such formulas are one kind which we can certainly say that we
know to be (essentially) true under $D$.

\begin{definition}[Witnessed formulas]
We define a formula to be {\em witnessed to evaluate to true or false} in a
partial assignment by induction on its construction; we say that the formula is 
{\em witnessed} iff it is witnessed to evaluate to either true or false.
\begin{compactitem}
\item A variable is witnessed to be true or false iff it is respectively
true or false in the partial assignment.
\item $\neg\phi$ is witnessed to evaluate to true iff $\phi$ is witnessed to
evaluate to false; naturally, $\neg\phi$ is witnessed to evaluate to false iff
$\phi$ is witnessed to evaluate to true.
\item A formula with a threshold connective $[c_1\phi_1+\cdots+
c_k\phi_k\geq b]$ is witnessed to evaluate to true iff
$\sum_{i:\phi_i\mathrm{\ witnessed\ true}}c_i
+\sum_{i:\phi_i\mathrm{\ not\ witnessed}}\min\{0,c_i\}
\geq b$
and it is witnessed to evaluate to false iff
$\sum_{i:\phi_i\mathrm{\ witnessed\ true}}c_i
+\sum_{i:\phi_i\mathrm{\ not\ witnessed}}\max\{0,c_i\}
<b.$
(i.e., iff the truth or falsehood, respectively, of the inequality is determined
by the witnessed formulas, regardless of what values are substituted for the
non-witnessed formulas.)
\end{compactitem}
\end{definition}

An example of particular interest is a CNF formula. A CNF is witnessed to
evaluate to true in a partial assignment precisely when every clause has some
literal that is satisfied. It is witnessed to evaluate to false precisely when 
there is some clause in which every literal is falsified.

Refining the motivating initial discussion somewhat, a witnessed formula is one
that can be evaluated in a very local manner. When the formula is {\em not}
witnessed, we will likewise be interested in the following ``simplification''
of the formula obtained from an incomplete evaluation:

\begin{definition}[Restricted formula]
Given a partial assignment $\rho$ and a formula $\phi$, the {\em restriction of 
$\phi$ under $\rho$,} denoted $\phi|_\rho$, is recursively defined as follows:
\begin{compactitem}
\item If $\phi$ is witnessed in $\rho$, then $\phi|_\rho$ is the formula
representing the value that $\phi$ is witnessed to evaluate to under $\rho$.
\item If $\phi$ is a variable not set by $\rho$, $\phi|_\rho=\phi$.
\item If $\phi=\neg\psi$ and $\phi$ is not witnessed in $\rho$, then $\phi|_\rho
=\neg(\psi|_\rho)$.
\item If $\phi=[\sum_{i=1}^kc_i\psi_i\geq b]$ and $\phi$ is not witnessed
in $\rho$, suppose that $\psi_1,\ldots,\psi_\ell$ are witnessed in $\rho$ (and
$\psi_{\ell+1},\ldots,\psi_k$ are not witnessed). Then $\phi|_\rho$ is
$[\sum_{i=\ell+1}^kc_i(\psi_i|_\rho)\geq d]$
where $d=b-\sum_{i:\psi_i|_\rho=1}c_i$.
\end{compactitem}
For a restriction $\rho$ and set of formulas $F$, we let $F|_\rho$ denote the
set $\{\phi|_\rho:\phi\in F\}$.
\end{definition}

\paragraph{Proof systems.}
We will need a formalization of a ``proof system'' in order to state our 
theorems:
\begin{definition}[Proof system]
A {\em proof system} is given by a sequence of relations $\{R_i\}_{i=0}^\infty$
over formulas such that $R_i$ is of arity-$(i+1)$ and whenever
$R_i(\psi_{j_1},\ldots,\psi_{j_i},\varphi)$ holds, ${\{\psi_{j_1},\ldots,
\psi_{j_i}\}\models\varphi}$. Any formula $\varphi$ satisfying $R_0$ is said to
be an {\em axiom} of the proof system. A {\em proof of a formula $\phi$} from
a set of {\em hypotheses} $H$ in the proof system is given by a finite sequence
of triples consisting of
\begin{compactenum}
\item A formula $\psi_k$
\item A relation $R_i$ of the proof system or the set $H$
\item A subsequence of formulas $\psi_{j_1},\ldots,\psi_{j_i}$ with $j_\ell<k$
for $\ell=1,\ldots,i$  (i.e., from the first components of earlier triples in
the sequence) such that $R_i(\psi_{j_1},\ldots,\psi_{j_i},\psi_k)$ holds, unless
$\psi_k\in H$.
\end{compactenum}
for which $\phi$ is the first component of the final triple in the sequence.
\end{definition}
Needless to say it is generally expected that $R_i$ is somehow efficiently
computable, so that the proofs can be checked. We don't explicitly impose such a
constraint on the formal object for the sake of simplicity, but the reader
should be aware that these expectations will be fulfilled in all cases of
interest.

We will be interested in the effect of the restriction (partial evaluation) 
mapping applied to {\em proofs}---that is, the ``projection'' of a proof in the 
original logic down to a proof over the smaller set of variables by the 
application of the restriction to every step in the proof. Although it may be 
shown that this at least preserves the (classical) semantic soundness of the 
steps, this falls short of what we require: we need to know that the {\em rules 
of inference} are preserved under restrictions. Since the relations defining the
proof system are arbitrary, though, this property must be explicitly verified. 
Formally, then:
\begin{definition}[Restriction-closed proof system]\label{res-closed-def}
We will say that a proof system over propositional formulas is {\em restriction
closed} if for every proof of the proof system and every partial assignment 
$\rho$, for any (satisfactory) step of the proof $R_k(\psi_1,\ldots,\psi_k,
\phi)$, there is some $j\leq k$ such that for the subsequence $\psi_{i_1},
\ldots,\psi_{i_j}$ $R_j(\psi_{i_1}|_\rho,\ldots,\psi_{i_j}|_\rho,\phi|_\rho)$ is
satisfied, and the formula $1$ (``true'') is an axiom.\footnote{%
This last condition is a technical condition that usually requires a trivial
modification of any proof system to accommodate. We can usually do without this
condition in actuality, but the details depend on the proof system.}
\end{definition}

So, when a proof system is restriction-closed, given a derivation of a formula
$\varphi$ from $\psi_1,\ldots,\psi_k$, we can extract a derivation of
$\varphi|_\rho$ from $\psi_1|_\rho,\ldots,\psi_k|_\rho$ for any partial 
assignment $\rho$ such that the steps of the proof consist of formulas 
mentioning only the variables masked in $\rho$. (In particular, we could think 
of this as a proof in a proof system for a logic with variables $\{x_i:\rho_i=
*\}$.) In a sense, this means that we can extract a proof of a ``special case''
from a more general proof by applying the restriction operator to every formula
in the proof. Again, looking ahead to Section~\ref{proof-systems}, we will see
that the typical examples of propositional proof systems that have been
considered essentially have this property.

We will be especially interested in limited versions of the decision problem for
a logic given by a collection of ``simple'' proofs---if the proofs are
sufficiently restricted, it is possible to give efficient algorithms to search
for such proofs, and then such a limited version of the decision problem will
be tractable, in contrast to the general case. Formally, now:
\begin{definition}[Limited decision problem]
Fix a proof system, and let $\calS$ be a set of proofs in the proof system.
The {\em limited decision problem for $\calS$} is then the following promise
problem: given as input a formula $\varphi$ with no free variables and a set
of hypotheses $H$ such that either there is a proof of $\varphi$ in $\calS$ from
$H$ or else $H\not\models\varphi$, decide which case holds.
\end{definition}
A classic example of such a limited decision problem for which efficient
algorithms exist is for formulas of propositional logic that have ``treelike''
resolution derivations of constant width (cf. the work of Ben-Sasson and
Wigderson~\cite{bsw01} or the work of Beame and Pitassi~\cite{bp96}, building on
work by Clegg et al.~\cite{cei96}). We will actually return to this example in
more detail in Section~\ref{proof-systems}, but we mention it now for the sake
of concreteness.

We will thus be interested in syntactic restrictions of restriction-closed
proof systems. We wish to know that (in contrast to the rules of the proof
system) these {\em syntactic restrictions} are likewise closed under
restrictions in the following sense:
\begin{definition}[Restriction-closed set of proofs]\label{res-closed-set-def}
A set of proofs $\calS$ is said to be {\em restriction closed} if whenever there
is a proof of a formula $\varphi$ from a set of hypotheses $H$ in $\calS$, there
is also a proof of $\varphi|_\rho$ in from the set $H|_\rho$ in $\calS$ for any
partial assignment $\rho$.
\end{definition}

\section{Inferences from incomplete data with implicit learning}
\label{incomplete-data-inferences}

A well-known general phenomenon in learning theory is that a restrictive choice
of representation for hypotheses often imposes artificial computational
difficulties. Since fitting a hypothesis is often a source of intractability, it
is natural to suspect that one would often be able to achieve more if the need
for such an explicit hypothesis were circumvented---that is, if ``learning''
were integrated more tightly into the application using the knowledge extracted
from data. For the application of answering queries, this insight was pursued by
Khardon and Roth~\cite{kr97} in the {\em learning to reason} framework, where
queries against an unknown DNF could be answered using examples.
The trivial algorithm that evaluates formulas on complete assignments
and uses the fraction satisfied to estimate the validity suggests how this might
happen: the examples themselves encode the needed information and so it is 
easier to answer the queries using the examples directly. 
In this case, the 
knowledge is used {\em implicitly:} the existence of the DNF describing the 
support of the distribution (thus, governing which models need to be considered)
guarantees that the behavior of the algorithm is correct, but at no point does 
the algorithm ``discover'' the representation of such a DNF. Effectively, we 
will develop an alternative approach that incorporates reasoning to cope with 
incomplete examples and explicit background knowledge, and yet retains the 
appealing circumvention of the construction of explicit representations for 
learned knowledge. In this approach, there are ``axioms'' that can be extracted 
from the observable data, which we suppose that if known, could be combined with
the background knowledge to answer a given query.

More formally, these ``axioms'' are formulas for which it is feasible to verify 
consistency with the underlying distribution (from the masked examples), that 
nevertheless suffice to complete a proof. This is necessary in some sense (cf. 
Proposition~\ref{axioms-checkable}), and at least seems to be not much more 
restrictive than the requirements imposed by concept learning. Specifically, we 
will utilize formulas that are witnessed to evaluate to true on the distribution
over partial assignments with probability at least $(1-\epsilon)$. We will 
consider any such formulas to be ``fair game'' for our algorithm, much as any 
member of a given concept class is ``fair game'' for concept learning.

We now state and prove the main theorem, showing that a variant of the limited
decision problem in which the proof may invoke these learnable formulas as 
``axioms'' is essentially no harder than the original limited decision problem, 
as long as the proof system is restriction-closed. The reduction is {\em very} 
simple and is given in Algorithm~\ref{pac-decision-alg}.

\begin{algorithm}[t]
\DontPrintSemicolon
\SetKwInOut{Input}{input}\SetKwInOut{Output}{output}
\SetKwInOut{Parameter}{parameter}

\Parameter{Algorithm $A$ solving the limited decision problem for the class of
proofs $\calS$.}
\Input{Formula $\varphi$, $\epsilon,\delta,\gamma\in (0,1)$, list of partial
assignments $\rho^{(1)},\ldots,\rho^{(m)}$ from $M(D)$, list of hypothesis 
formulas $H$}
\Output{{\em Accept} if there is a proof of $\varphi$ in $\calS$ from $H$ and
formulas $\psi_1,\psi_2,\ldots$ that are simultaneously witnessed true with
probability at least $1-\epsilon+\gamma$ on $M(D)$;\\
{\em Reject} if $H\Rightarrow\varphi$ is not $(1-\epsilon-\gamma)$-valid under
$D$.}

\Begin{
$B\leftarrow \lfloor\epsilon\cdot m\rfloor$, $FAILED\leftarrow 0$.\\
\ForEach{partial assignment $\rho^{(i)}$ in the list}{
  \If{$A(\varphi|_{\rho^{(i)}},H|_\rho)$ rejects}{
    Increment $FAILED$.
    \If{$FAILED>B$}{\Return{{\em Reject}}}
    }
  }
\Return{{\em Accept}}
}

\caption{DecidePAC}\label{pac-decision-alg}
\end{algorithm}

\begin{theorem}[Adding implicit learning preserves tractability]
\label{implicit-learn-thm}
Let $\calS$ be a restriction-closed set of proofs for a restriction-closed proof
system. Suppose that there is an algorithm for the limited decision problem for
$\calS$ running in time $T(n,|\varphi|,|H|)$ on input $\varphi$ and $H$ over $n$
variables. Let $D$ be a distribution over assignments, $M$ be any masking 
process, and $H$ be any set of formulas.
Then there is an algorithm that, on input $\varphi$, $H$, $\delta$ and
$\epsilon$, uses $O(1/\gamma^2\log 1/\delta)$ examples, runs in
time $O(T(n,|\varphi|,|H|)\frac{1}{\gamma^2}\log\frac{1}{\delta})$, and such
that given that either
\begin{compactitem}
\item $[H\Rightarrow\varphi]$ is not $(1-\epsilon-\gamma)$-valid with respect to
$D$ or
\item there exists a proof $\varphi$ from $\{\psi_1,\ldots,\psi_k\}\cup H$ in
$\calS$ such that $\psi_1,\ldots,\psi_k$ are all witnessed to evaluate to
true with probability $(1-\epsilon+\gamma)$ over $M(D)$
\end{compactitem}
decides which case holds.
\end{theorem}
\begin{proof}
Suppose we run Algorithm~\ref{pac-decision-alg} on
$m=\frac{1}{2\gamma^2}\ln\frac{1}{\delta}$ examples drawn from $D$. Then,
(noting that we need at most $\log m$ bits of precision for $B$) the claimed
running time bound and sample complexity is immediate.

As for correctness, first note that by the soundness of the proof system,
whenever there is a proof of $\varphi|_{\rho^{(i)}}$ from $H|_{\rho^{(i)}}$,
$\varphi|_{\rho^{(i)}}$ must evaluate to true in any interpretation of the
remaining variables consistent with $H|_{\rho^{(i)}}$. Thus, if $H\Rightarrow
\varphi$ is {\em not} $(1-\epsilon-\gamma)$-valid with respect to $D$, an 
interpretation sampled from $D$ must satisfy $H$ and falsify $\varphi$
with probability at least $\epsilon+\gamma$; for any partial assignment $\rho$
derived from this interpretation (i.e., sampled from $M(D)$), the original
interpretation is still consistent, and therefore $H|_\rho\not\models
\varphi|_\rho$ for this $\rho$. So in summary, we see that a $\rho$ sampled from
$M(D)$ produces a formula $\varphi|_\rho$ such that $H|_\rho\not\models
\varphi|_\rho$ with probability at least $\epsilon+\gamma$, and so the limited
decision algorithm $A$ rejects with probability at least $\epsilon+\gamma$. It
follows from Hoeffding's inequality now that for $m$ as specified above, at
least $\epsilon m$ of the runs of $A$ reject (and hence the algorithm rejects)
with probability at least $1-\delta$.

So, suppose instead that there is a proof in $\calS$ of $\varphi$ from $H$ and
some formulas $\psi_1,\ldots,\psi_k$ that are all witnessed to evaluate to true
with probability at least $(1-\epsilon+\gamma)$ over $M(D)$. Then, with
probability $(1-\epsilon+\gamma)$, $\psi_1|_\rho,\ldots,\psi_k|_\rho=1$. Then,
since $\calS$ is a restriction closed set, if we replace each assertion of some
$\psi_j$ with an invocation of $R_0$ for the axiom $1$, then by applying the
restriction $\rho$ to every formula in the proof, one can obtain a proof of
$\varphi|_\rho$ from $H|_\rho$ alone. Therefore, as $A$ solves the limited
decision problem for $\calS$, we see that for each $\rho$ drawn from $M(D)$,
$A(\varphi|_\rho,H|_\rho)$ must accept with probability at least $(1-\epsilon+
\gamma)$, and Hoeffding's inequality again gives that the probability that more
than $\epsilon m$ of the runs reject is at most $\delta$ for this choice of $m$.
\end{proof}

\paragraph{The necessity of computationally feasible witnessing.}
The reader may, at this point, feel that our notion of witnessed values is
somewhat ad-hoc, and suspect that perhaps a weaker notion should be considered
(corresponding to a broader class of masking processes). Although it may be the
case that a better notion exists, we observe in Appendix~\ref{feasible-witness} 
that it is crucial that we use {\em some} kind of evaluation algorithm on 
partial assignments that is computationally feasible. Witnessed evaluation is 
thus, at least, one such notion, whereas other natural notions are likely 
computationally infeasible, and thus inappropriate for such purposes.

\section{Proof systems with tractable, restriction-closed special cases}
\label{proof-systems}
We now show that most of the usual propositional proof systems considered in
the literature possess natural restriction-closed special cases, for which
the limited decision problem may be efficiently solved. Thus, in each case, we
can invoke Theorem~\ref{implicit-learn-thm} to show that we can efficiently
integrate implicit learning into the reasoning algorithm for the proof system.

\subsection{Special cases of resolution}\label{res-sec}

Our first example of a proof system for use in reasoning in PAC-Semantics is 
{\em resolution}, a standard object of study in proof theory. Largely due to its
simplicity, resolution turned out to be an excellent system for the design of 
surprisingly effective proof search algorithms such as DPLL~\cite{dp60,dll62}. 
Resolution thus remains attractive as a proof system possessing natural special
cases for which we can design relatively efficient algorithms for proof search. 
We will recall two such examples here.

\paragraph{The resolution proof system.}
Resolution is a proof system that operates on {\em clauses}---disjunctions of
literals. The main inference rule in resolution is the {\em cut} rule: given
two clauses containing a {\em complementary pair} of literals (i.e., one
contains the negation of a variable appearing without negation in the other)
$A\vee x$ and $B\vee \neg x$, we infer the {\em resolvent} $A\vee B$. We will
also find it convenient to use the {\em weakening} rule: from any clause $C$,
for any set of literals $\ell_1,\ldots,\ell_k$, we can infer the clause $C\vee
\ell_1\vee\cdots\vee\ell_k$.
As stated, resolution derives new clauses from a set of known clauses (a CNF
formula). Typically, one actually refers to resolution as a proof system for
{\em DNF formulas} by using a resolution proof as a proof by contradiction: one
shows how the unsatisfiable empty clause $\bot$ can be derived from the negation
of the input DNF. This is referred to as a {\em resolution refutation} of the
target DNF, and can also incorporate explicit hypotheses given as CNF formulas.

\paragraph{Treelike resolution proofs.}
The main syntactic restriction we consider on resolution refutations intuitively
corresponds to a restriction that a clause has to be derived anew each time we
wish to use it in a proof---a restriction that the proof may not (re-)use
``lemmas.'' It will not be hard to see that while this does not impact the
completeness of the system since derivations may be repeated, this workaround
comes at the cost of increasing the size of the proof.
A syntactic way of capturing these proofs proceeds by recalling that the
proof is given by a sequence of clauses that are either derived from earlier
clauses in the sequence, or appear in the input CNF formula (to be refuted).
Consider the following directed acyclic graph (DAG) corresponding to any
(resolution) proof: the set of nodes of the graph is given by the set of clauses
appearing in the lines of the proof, and each such node has incoming edges from
the nodes corresponding to the clauses earlier in the proof used in its
derivation; the clauses that appeared in the input CNF formula are therefore
the sources of this DAG, and the clause proved by the derivation corresponds to
a sink of the DAG (i.e., in a resolution refutation, the empty clause appears
at a sink of the DAG). We say that the proof is {\em treelike} when this DAG is
a {\em (rooted) tree}---i.e., each node has at most one outgoing edge
(equivalently, when there is a unique path from any node to the unique sink).
Notice, the edges correspond to the use of a clause in a step of the proof, so
this syntactic restriction corresponds to our intuitive notion described
earlier.

We are interested in resolution as a proof system with special cases that not 
only possess efficient decision algorithms, but are furthermore 
restriction-closed. We will first establish that (treelike) resolution in 
general is restriction-closed, and subsequently consider the effects of our 
additional restrictions on the proofs considered.
For syntactic reasons (to satisfy Definition~\ref{res-closed-def}), actually, we
need to include a tautological formula $1$ as an axiom of resolution. We can 
take this to correspond to the clause containing all literals, which is always 
derivable by weakening from any nonempty set of clauses (and is furthermore 
essentially useless in any resolution proof, as it can only be used to derive 
itself).

\begin{proposition}[Treelike resolution is restriction-closed]
\label{treelike-res-closed}
Resolution is a restriction-closed proof system. Moreover, the set of
treelike resolution proofs of length $L$ is restriction-closed.
\end{proposition}
\begin{proof}
Assuming the inclusion of the tautological axiom $1$ as discussed above,
the restriction-closedness is straightforward: Fix an partial assignment $\rho$,
and consider any step of the proof, deriving a clause $C$. If $C$ appeared in 
the input formula, then $C|_\rho$ appears in the restriction of the input 
formula. Otherwise, $C$ is derived by one of our two rules, cut or weakening. 
For the cut rule, suppose $C$ is derived from $A\vee x_i$ and $B\vee\neg x_i$. 
If $\rho_i\in\{0,1\}$ then $C$ can either be derived from
$(A\vee x_i)|_\rho$ or $(B\vee\neg x_i)|_\rho$ by weakening. If
$\rho_i=*$ and $C|_\rho\neq 1$, then both $(A\vee x_i)|_\rho$ and
$(B\vee\neg x_i)|_\rho$ are not $1$, and the same literals are eliminated (set
to $0$) in these clauses as in $C|_\rho$, so $C|_\rho$ follows from the cut rule
applied to $x_i$ on these clauses. If $C|_\rho\neq 1$ followed from weakening
of some other clause $C'$, we know $C'|_\rho\neq 1$ as well, since any satisfied
literals in $C'$ appear in $C$; therefore $C|_\rho$ follows from weakening
applied to $C'|_\rho$. Finally, if $C|_\rho=1$, then we already know that $1$
can be asserted as an axiom.  So, resolution is restriction-closed.

Recalling the DAG corresponding to a resolution proof has nodes corresponding
to clauses and edges indicating which clauses are used in the derivation of
which nodes, note that the DAG corresponding to the restriction of a resolution
proof as constructed in the previous paragraph has no additional edges.
Therefore, the sink in the original DAG remains a sink. Although the DAG may now
be disconnected, if consider the connected component containing the node
corresponding to the original sink, we see that this is indeed a tree;
furthermore, since every clause involved in the derivation of a clause
corresponding to a node of the tree corresponds to another node of the tree and
the overall DAG corresponded to a syntactically correct resolution proof from
the restriction of the input formula, by the restriction-closedness of
resolution, this tree corresponds to a treelike resolution proof of the
restriction of the clause labeling the sink from the restriction of the input
formula. As this is a subgraph of the original graph, it corresponds to a proof
that is also no longer than the original, as needed.
\end{proof}

\paragraph{Bounded-space treelike resolution.}
Our first special case assumes not only that the resolution proof is treelike,
but also that it can be carried out using limited {\em space}, in the sense
first explored by Esteban and Tor\'an~\cite{et01}. That is, we associate with
each step of the proof a set of clauses that we refer to as the
{\em blackboard}. Each time a clause is derived during a step of the proof, we
consider it to be added to the blackboard; we also allow any clauses in the
blackboard to be erased across subsequent steps of the proof. Now, the central
restriction is that instead of simply requiring the steps of the proof to
utilize clauses that appeared earlier in the proof, we demand that they {\em
only utilize clauses that appeared in the blackboard set on the previous step.}
We now say that the proof uses {\em (clause) space $s$} if the blackboard never
contains more than $s$ clauses. We note that the restriction that the proof is 
treelike means that each time we utilize clauses in a derivation, we are free to
delete them from the blackboard. In fact, given the notion of a blackboard, it 
is easily verified that this is an equivalent definition of a treelike proof.
Even with the added restriction to clause space $s$, treelike resolution remains
restriction-closed:

\begin{proposition}\label{space-res-closed}
The set of clause space-$s$ treelike resolution proofs is restriction closed.
\end{proposition}
\begin{proof}
Let a space-$s$ treelike resolution proof $\Pi$ and any partial assignment 
$\rho$ be given; we recall the corresponding treelike proof $\Pi'$ constructed 
in the proof of Proposition~\ref{treelike-res-closed}; we suppose that $\Pi$ 
derives the sequence of clauses $\{C_i\}_{i=1}^{|\Pi|}$ (for which $C_i$ is 
derived on the $i$th step of $\Pi$) and $\Pi'$ derives the subsequence
$\{C_{i_j}|_\rho\}_{j=1}^{|\Pi'|}$. Given the corresponding sequence of
blackboards $\{B_i\}_{i=1}^{|\Pi|}$ establishing that $\Pi$ can be carried out
in clause space $s$, we construct a sequence of blackboards $B'_i=\{C_j|_\rho:
C_j\in B_i,\exists k\mathrm{\ s.t.\ }j=i_k\}$ for $\Pi'$, and take the
subsequence corresponding to steps in $\Pi'$, $\{B'_{i_j}\}_{j=1}^{|\Pi'|}$.

It is immediate that every $B'_{i_j}$ contains at most $s$ clauses, so we only
need to establish that these are a legal sequence of blackboards for $\Pi'$. We
first note that whenever a clause is added to a blackboard $B'_{i_j}$ over
$B'_{i_{j-1}}$, then since (by construction) it was not added in $i'\in
[i_{j-1},i_j]$ it must be that it is added (to $B_{i_j}$) in step $i_j$, which
we know originally derived $C_{i_j}$ in $\Pi$, and hence in $\Pi'$ derives
$C_{i_j}|_\rho$ by construction of $\Pi'$ (so this is the corresponding $j$th
step of $\Pi'$). Likewise, if a clause is needed for the derivation of any $j$th
step of $\Pi'$, by the construction of $\Pi'$ from $\Pi$, it must be that
$C_{i_j}|_\rho\neq 1$ and whenever some step $i_j$ of $\Pi$ uses an unsatisfied
clause from some earlier step $t$ of $\Pi$, then $\Pi'$ includes the step
corresponding to $t$. Therefore there exists $k$ such that $t=i_k$; and, as
$C_{i_k}\in B_{i_j}$, $C_{i_k}|_\rho\in B'_{i_j}$. Thus,
$\{B'_{i_j}\}_{j=1}^{|\Pi'|}$ is a legal sequence of blackboards for $\Pi'$.
\end{proof}

The algorithm for finding space-$s$ resolution proofs, SearchSpace, appears as
Algorithm~\ref{search-space}. Although the analysis of this algorithm appears 
elsewhere, we include the proof (and its history) in Appendix~\ref
{space-bound-alg-appendix} for completeness.

\begin{algorithm}[t]
\DontPrintSemicolon
\SetKwFunction{SS}{SearchSpace}
\SetKwInOut{Input}{input}\SetKwInOut{Output}{output}

\Input{CNF $\varphi$, integer space bound $s\geq 1$, current clause $C$}
\Output{A space-$s$ treelike resolution proof of $C$ from clauses in $\varphi$,
or ``none'' if no such proof exists.}

\Begin{
\If{$C$ is a superset of some clause $C'$ of $\varphi$}{
\Return{The weakening derivation of $C$ from $C'$.}}
\ElseIf{$s>1$}{
\ForEach{Literal $\ell$ such that neither $\ell$ nor $\neg \ell$ is in $C$}{
\If{$\Pi_1\leftarrow$\SS$(\varphi,s-1,C\vee \ell)$ does not return {\em none}}{
\If{$\Pi_2\leftarrow$\SS$(\varphi,s,C\vee\neg\ell)$ does not return {\em none}}{
\Return{Derivation of $C$ from $\Pi_1$ and $\Pi_2$}
}
\Else{
\Return{{\em none}}
}
}
}
}
\Return{{\em none}}
}
\caption{SearchSpace}\label{search-space}
\end{algorithm}

\begin{theorem}[SearchSpace finds space-$s$ treelike proofs when they exist]
\label{space-analysis}
If there is a space-$s$ treelike proof of a clause $C$ from a CNF formula
$\varphi$, then SearchSpace returns such a proof, and otherwise it returns
``none.'' In either case, it runs in time $O(|\varphi|\cdot n^{2(s-1)})$ where
$n$ is the number of variables.
\end{theorem}

Naturally, we can convert SearchSpace into a decision algorithm by accepting
precisely when it returns a proof. Therefore, as space-$s$ treelike resolution
proofs are restriction-closed by Proposition~\ref{space-res-closed},
Theorem~\ref{implicit-learn-thm} can be applied to obtain an algorithm that
efficiently learns implicitly from example partial assignments to solve the
corresponding limited decision problem for $(1-\epsilon)$-validity with
space-$s$ treelike resolution proofs. Explicitly, we obtain:

\begin{corollary}[Implicit learning in space-bounded treelike resolution]
\label{space-implicit-cor}
Let a KB CNF $\phi$ and clause $C$ be given, and suppose that partial 
assignments are drawn from a masking process for an underlying distribution $D$;
suppose further that either
\begin{compactenum}
\item There exists some CNF $\psi$ such that partial assignments from the 
masking process are witnessed to satisfy $\psi$ with probability at least $(1-
\epsilon+\gamma)$ and there is a space-$s$ treelike proof of $C$ from $\phi
\wedge\psi$
or else
\item $[\phi\Rightarrow C]$ is at most $(1-\epsilon-\gamma)$-valid with respect
to $D$ for $\gamma>0$.
\end{compactenum}
Then, there an algorithm running in time
$O(\frac{|\phi|}{\gamma^2}n^{2(s-1)}\log\frac{1}{\delta})$
that distinguishes these cases with probability $1-\delta$
when given $C$, $\phi$, $\epsilon$, $\gamma$, and a
sample of $O(\frac{1}{\gamma^2}\log\frac{1}{\delta})$ partial assignments.
\end{corollary}

\paragraph{A quasipolynomial time algorithm for treelike resolution.}
As we noted previously, Beame and Pitassi~\cite{bp96} gave an algorithm
essentially similar to SearchSpace, but only established that it could find
treelike proofs in quasipolynomial time. Their result follows from Theorem~\ref
{space-analysis} and the following generic space bound:

\begin{proposition}\label{treelike-log-space}
A treelike proof $\Pi$ can be carried out in clause space at most $\log_2
|\Pi|+1$.
\end{proposition}

So therefore, if there is a treelike proof of a clause $C$ from a formula
$\varphi$ of size $n^k$, SearchSpace (run with the bound $s=k\log n+1$)
finds the proof in time $O(|\varphi|\cdot n^{2k\log n})$. We also include the
proof in Appendix~\ref{space-bound-alg-appendix}.

\paragraph{Bounded-width resolution.}
\iffalse{
\begin{algorithm}[t]
\DontPrintSemicolon
\SetKwInOut{Input}{input}\SetKwInOut{Output}{output}

\Input{CNF formula $\varphi$, width bound $w\in\bbN$.}
\Output{{\em Accept} if there is a resolution refutation of $\varphi$ of
width $w$; {\em Reject} otherwise.}

\Begin{
Initialize a table $T[C]\leftarrow 0$ for every clause $C$ of width at most
$w$.\\
$NEW\leftarrow 1$.\\
\While{$NEW=1$}{
 $NEW\leftarrow 0$.\\
\ForEach{Atomic formula $\alpha$}{
 \ForEach{Pair of clauses $(C_1\vee\alpha,C_2\vee\neg\alpha)$ of width at most
$w$ or in $\varphi$}{
    \If{$C_1\vee C_2$ has width at most $w$ and $T[C_1\vee C_2]=0$}{
      \If{$C_1\vee C_2=\bot$}{\Return {\em Accept}}
      $NEW\leftarrow 1$; $T[C_1\vee C_2]\leftarrow 1$}
   }
 }
}
\Return{ {\em Reject} }

}

\caption{DecideWidth}\label{width-decision-alg}
\end{algorithm}

}\fi
Our second special case of resolution considers proofs using small clauses.
Precisely, we refer to the number of literals appearing in a clause as the
{\em width} of the clause, and we naturally consider the {\em width of a
resolution proof} to be the maximum width of any clause derived in the proof
(i.e., excluding the input clauses). Bounded-width resolution was originally
formally investigated by Galil~\cite{galil77}, who exhibited an efficient
dynamic programming algorithm for bounded-width resolution. Galil's algorithm
easily generalizes to $k$-DNF resolution, i.e., the proof system {\em RES$(k)$},
(with standard resolution being recovered by $k=1$) so we will present the more
general case here.

Briefly, {\em RES$(k)$}, introduced by Kraj\'{\i}\v{c}ek~\cite{krajicek01}, is a
proof system that generalizes resolution by operating on $k$-DNF formulas
instead of clauses (which are, of course, $1$-DNF formulas) and introduces some
new inference rules, described below. In more detail, recall that a $k$-DNF is a
disjunction of conjunctions of literals, where each conjunction contains at
most $k$ literals. Each step of a RES$(k)$ proof derives a $k$-DNF from one of
the following rules. {\em Weakening} is essentially similar to the analogous
rule in resolution: from a $k$-DNF $\varphi$, we can infer the $k$-DNF $\varphi
\vee\psi$ for any $k$-DNF $\psi$. RES$(k)$ also features an essentially similar
{\em cut} rule: from a $k$-DNF $A\vee (\ell_1\wedge\cdots\wedge\ell_j)$ ($j\leq 
k$) and another $k$-DNF $B\vee \neg\ell_1\vee\cdots\vee\neg\ell_j$, we can infer
the $k$-DNF $A\vee B$. The new rules involve manipulating the conjunctions: 
given $j\leq k$ formulas $\ell_1\vee A,\ldots,\ell_j\vee A$, we can infer 
$(\ell_1\wedge\cdots\wedge\ell_j)\vee A$ by {\em $\wedge$-introduction}. 
Likewise, given $(\ell_1\wedge\cdots\wedge\ell_j)\vee A$, we can infer $\ell_i
\vee A$ for any $i=1,\ldots,j$ by {\em $\wedge$-elimination}.

We wish to show that RES$(k)$ is restriction-closed; actually, for technical
simplicity, we will represent $1$ by the disjunction of all literals.
This can be derived from any DNF by a linear number of $\wedge$-elimination
steps (in the size of the original DNF) followed by a weakening step, so it is
not increasing the power of RES$(k)$ appreciably to include such a rule.

\begin{proposition}\label{res-k-res-closed}
For any $k$, RES$(k)$ is restriction-closed.
\end{proposition}
\begin{proof}
We are given (by assumption) that our encoding of $1$ is an axiom. Let any 
partial assignment $\rho$ be given, and consider the DNF $\varphi$ derived
on any step of the proof. Naturally, if $\varphi$ was a hypothesis, then
$\varphi|_\rho$ is also a hypothesis. Otherwise, it was derived by one of the
four inference rules. We suppose that $\varphi|_\rho\neq 1$ (or else we are
done). Thus, if $\varphi$ was derived by weakening from $\psi$, it must be the
case that $\psi|_\rho\neq 1$, since otherwise $\varphi|_\rho=1$, so
$\varphi|_\rho$ follows from $\psi|_\rho$ again by weakening since every
conjunction in $\psi|_\rho$ appears in $\varphi|_\rho$. Likewise, if $\varphi=
\ell_i\vee A$ was derived by $\wedge$-elimination from $\psi=(\ell_1\wedge\cdots
\wedge\ell_j)\vee A$, then since $\ell_i|_\rho\neq 1$ and $A|_\rho$ must not be
$1$, neither the conjunction $\ell_i$ was taken from in $\psi$ nor the rest of
the formula $A$ evaluates to $1$ and thus $\psi|_\rho\neq 1$. Then, if some
$\ell_t$ is set to $0$ by $\rho$, $\psi|_\rho=A|_\rho$, and $\varphi|_\rho$
follows from $\psi|_\rho$ by weakening; otherwise, $\varphi|_\rho$ still follows
by $\wedge$-elimination.

We now turn to consider $\varphi=(\ell_1\wedge\cdots\wedge\ell_j)\vee A$ that
were derived by $\wedge$-introduction. We first consider the case where some
literal $\ell_i$ in the new conjunction is set to $0$ in $\rho$ (and so
$\varphi|_\rho=A|_\rho$). In this case, one of the premises in the
$\wedge$-introduction step was $\ell_i\vee A$, where $(\ell_i\vee A)|_\rho =
A|_\rho = \varphi|_\rho$, so in fact $\varphi|_\rho$ can be derived just as
$\ell_i\vee A$ was derived. We now suppose that no $\ell_i$ is set to $0$ in
$\rho$; let $\ell_{i_1},\ldots,\ell_{i_s}$ denote the subset of those literals
that are not set to $1$ (i.e., satisfy $\ell_{i_t}|_\rho=\ell_{i_t}$). Then
$\varphi|_\rho=(\ell_{i_1}\wedge\cdots\wedge\ell_{i_s})\vee A|_\rho$, where
since $A|_\rho\neq 1$, the premises $\ell_{i_t}\vee A$ used to derive $\varphi$
all satisfy $(\ell_{i_t}\vee A)|_\rho=\ell_{i_t}\vee A|_\rho\neq 1$, and so we
can again derive $\varphi|_\rho$ by $\wedge$-introduction from this subset of
the original premises.

Finally, we suppose that $\varphi=A\vee B$ we derived by the cut rule applied
to $A\vee (\ell_1\wedge\cdots\wedge\ell_j)$ and $B\vee\neg\ell_1\vee\cdots\vee
\neg\ell_j$. If some $\ell_i$ is set to $0$ by $\rho$, then the first premise
satisifies $(A\vee (\ell_1\wedge\cdots\wedge\ell_j))|_\rho=A|_\rho$ and so
$\varphi|_\rho=A|_\rho\vee B|_\rho$ can be derived by weakening from the first
premise. If not, we let $\ell_{i_1},\ldots,\ell_{i_s}$ denote the subset of
those literals that are not set to $1$. Then the first premise becomes
$A|_\rho\vee (\ell_{i_1}\wedge\cdots\wedge\ell_{i_s})\neq 1$ (since we assumed
$\varphi|_\rho\neq 1$) and likewise, the second premise becomes $B|_\rho\vee\neg
\ell_{i_1}\vee\cdots\vee\neg\ell_{i_s}\neq 1$ (as likewise $B|_\rho\neq 1$ and
no $\ell_{i_t}|_\rho=0$), so $\varphi|_\rho$ follows by the cut rule applied to
these two premises.
\end{proof}

Now, RES$(k)$ possesses a ``bounded-width'' restriction for which we will
observe has a limited decision problem that can be solved by a dynamic 
programming algorithm (given in pseudocode as Algorithm~\ref
{res-k-width-decision-alg}). More precisely, we will say that a DNF has 
{\em width $w$} if it is a disjunction of at most $w$ conjunctions, and so 
likewise the {\em width of a RES$(k)$ proof} is the maximum width of any 
$k$-DNF derived in the proof.

\begin{algorithm}[f]
\DontPrintSemicolon
\SetKwInOut{Input}{input}\SetKwInOut{Output}{output}

\Input{List of $k$-DNF formulas $\varphi_1\ldots,\varphi_\ell$, target width-$w$
$k$-DNF $\phi$, width bound $w\in\bbN$.}
\Output{{\em Accept} if there is a RES$(k)$ proof of $\phi$ of
width $w$; {\em Reject} otherwise.}

\Begin{
Initialize a table $T[\psi]\leftarrow 0$ for every $k$-DNF $\psi$ of width at
most $w$ and then set $T[\varphi_i]\leftarrow 1$ for each $\varphi_i$ that is a
width-$w$ $k$-DNF.\\
$NEW\leftarrow 1$.\\
\While{$NEW=1$}{
 \If{$T[\phi]=1$}{
   \Return{ {\em Accept} }
 }
 $NEW\leftarrow 0$.\\
 \ForEach{$k$-DNF $\psi_1$ of width at most $w$ with $T[\psi_1]=1$ or among
    $\varphi_1,\ldots,\varphi_\ell$}{
    \ForEach{Formula $\psi'$ of width at most $w$ derivable from $\psi_1$ by
      weakening or $\wedge$-elimination}{
      \If{$T[\psi']=0$}{
        $T[\psi']\leftarrow 1$; $NEW\leftarrow 1$
      }
    }
    \ForEach{Formula $\psi_2$ of width at most $w$ with $T[\psi_2]=1$ or among
      $\varphi_1,\ldots,\varphi_\ell$}{
      \If{The cut rule can be applied to $\psi_1$ and $\psi_2$ yielding a
        $k$-DNF $\psi'$ of width at most $w$}{
        $T[\psi']\leftarrow 1$; $NEW\leftarrow 1$
      }
    }
 }
 \ForEach{$j$-tuple of distinct $k$-DNFs $(\psi_1,\ldots,\psi_j)$ of width $w$
   with $T[\psi_i]=1$ (for $i=1,\ldots,j$) with $j\leq k$}{
   \If{$\wedge$-introduction can be applied to $\psi_1,\ldots,\psi_j$, yielding
     a width-$w$ $k$-DNF $\psi'$}{
      $T[\psi']\leftarrow 1$; $NEW\leftarrow 1$
    }
 }
}
\Return{ {\em Reject} }

}

\caption{Pseudocode for Decide-RES(k)-Width}\label{res-k-width-decision-alg}
\end{algorithm}

\begin{theorem}[Efficient decision of bounded-width RES(k)]
\label{res-k-width-analysis}
Algorithm~\ref{res-k-width-decision-alg} accepts iff there is a RES$(k)$ proof
of its input $\phi$ from the input $k$-DNF formulas $\varphi_1\ldots,
\varphi_\ell$ of width at most $w$. If there are $n$ variables, it runs in
time $O(n^{kw+1}(n^{kw}+\ell)^k\max\{kn^{kw},|\varphi_i|\})$.
\end{theorem}
\begin{proof}
The correctness is straightforward: if there is a width-$w$ RES$(k)$ proof, then
a new derivation step from the proof is performed on each iteration of the main 
loop until $\phi$ is derived, and conversely, every time $T[\psi]$ is set to 
$1$, a width-$w$ derivation of $\psi$ could be extracted from the run of the 
algorithm. So, it only remains to consider the running time.

The main observation is that there are at most $O(n^{kw})$ width-$w$ $k$-DNFs.
(The initialization thus takes time at most $O(n^{kw}\ell)$.)
At least one of these must be derived on each iteration. Each iteration 
considers all possible derivations using up to $k$ distinct formulas either in 
the table or given in the input, of which there are $O((n^{kw}+\ell)^k)$ tuples.
We thus need to consider only the time to check each of the possible 
derivations.

A formula $\psi_1$ must be a width-$w$ $k$-DNF for another width-$w$ $k$-DNF
$\psi'$ to be derivable via weakening, and then for each other width-$w$ $k$-DNF
$\psi'$, we can check whether or not it is a weakening of $\psi_1$ in time
$O(n^{kw})$ by just checking whether all of the conjunctions of $\psi_1$ appear
in $\psi'$. Likewise, for $\wedge$-introduction, the formula must already be
a width-$w$ $k$-DNF, and we can check whether or not the $j\leq k$ formulas have
a shared common part by first checking which conjunctions from the first formula
appear in the second, and then, if only one literal is left over in each,
checking that the other $j-2$ formulas have the same common parts with one
literal left over. We then obtain the resulting derivation by collecting these
$j$ literals, in an overall time of $O(kn^{kw})$.

For the $\wedge$-elimination rule, the formula must already be width-$w$ for us
to obtain a width-$w$ result. Then, we can easily generate each of the possible
results in time linear in the length of the formula, that is, $O(n^{kw})$.
For the cut rule, we only need to examine each conjunction of each formula,
and check if the literals appear negated among the conjunctions of the other
formula, taking time linear in the size of the formulas, which is
$O(\max\{n^{kw},|\varphi_i|\})$; checking that the result is a width-$w$ $k$-DNF
then likewise can be done in linear time in the size of the formulas.
\end{proof}

Finally, we note that the width-$w$ syntactic restriction of RES$(k)$
refutations is restriction-closed:

\begin{proposition}\label{width-w-res-k-res-closed}
The set of width-$w$ RES$(k)$ refutations is restriction-closed.
\end{proposition}
\begin{proof}
Let any width-$w$ RES$(k)$ refutation $\Pi$ and partial assignment $\rho$ be 
given. In the construction used in Proposition~\ref{res-k-res-closed}, we 
obtained a proof $\Pi'$ of $\bot|_\rho=\bot$ from $\Pi$ with the property that 
every formula $\psi'$ appearing in $\Pi'$ satisfies $\psi'=\psi|_\rho$ for some
$\psi$ appearing in $\Pi$. Furthermore, we guaranteed that no derivation step
used a formula that simplified to $1$. It therefore suffices to note that for
any width-$w$ $k$-DNF $\psi$, $\psi|_\rho$ is also a $k$-DNF with width at most
$w$.
\end{proof}

By Theorem~\ref{implicit-learn-thm}, DecidePAC can be applied to Algorithm~\ref
{res-k-width-decision-alg} to obtain a second implicit learning algorithm,
for a width-$w$ RES$(k)$.

\begin{corollary}[Implicit learning in bounded-width RES(k)]
\label{res-k-width-implicit-cor}
Let a KB of $k$-DNFs $\phi_1\ldots,\phi_\ell$ and target disjunction of $k$-CNFs
$\varphi$ be given, and suppose that partial assignments are drawn from a 
masking process for an underlying distribution $D$; suppose further that either
\begin{compactenum}
\item There exists some conjunction of $k$-DNFs $\psi$ such that partial 
assignments from the masking process are witnessed to satisfy $\psi$ with 
probability at least $(1-\epsilon+\gamma)$ and there is a width-$w$ RES$(k)$ 
refutation of $\neg\varphi\wedge\phi_1\wedge\cdots\wedge\phi_\ell\wedge\psi$ 
or else
\item $[\phi_1\wedge\cdots\wedge\phi_\ell\Rightarrow\varphi]$ is at most
$(1-\epsilon-\gamma)$-valid with respect to $D$ for $\gamma>0$.
\end{compactenum}
Then, there an algorithm running in time
$O(n^{kw+1}(n^{kw}+\ell)^k\max\{kN^{kw},|\phi_i|\}\frac{1}{\gamma^2}
\log\frac{1}{\delta})$
that distinguishes these cases with probability $1-\delta$
when given $\varphi$, $\phi_1\ldots,\phi_\ell$, $\epsilon$, $\gamma$, and a
sample of $O(\frac{1}{\gamma^2}\log\frac{1}{\delta})$ partial assignments.
\end{corollary}

\subsection{Degree-bounded polynomial calculus}\label{pc-sec}

Our next example proof system is {\em Polynomial calculus}, an algebraic proof
system originally introduced by Clegg et al.~\cite{cei96} as a (first) example
of a proof system that could simulate resolution (the gold standard for
theorem-proving heuristics) on the one hand, and possessing a natural special
case for which the limited decision problem could demonstrably be solved in
polynomial time using a now standard computer algebra algorithm, the
{\em Gr\"{o}bner basis algorithm} due to Buchberger~\cite{buchberger85}.
Although the original hopes of Clegg et al. -- that polynomial calculus might
one day supplant resolution as the proof system of choice -- have not been
fulfilled due to the fact that heuristics based on resolution have been observed
to perform spectacularly well in practice, it nevertheless represents a
potentially more powerful system that furthermore alludes to the diversity
possible among proof systems.

\paragraph{The polynomial calculus proof system.}
In polynomial calculus, formulas have the form of polynomial equations over
an arbitrary nontrivial field $\bbF$ (for the present purposes, assume $\bbF$
is $\bbQ$, the field of rationals), and we are interested in their Boolean
solutions. A set of hypotheses is thus a system of equations, and polynomial
calculus enables us to derive new constraints that are satisfied by any Boolean
solutions to the original system. Of course, in this correspondence, our
Boolean variables serve as the variables of the polynomials.

More formally, for our Boolean variables $x_1,\ldots,x_n$, our
formulas are equations of the form $[p=0]$ for $p\in\bbF[x_1,\ldots,
x_n]$ (i.e., formal multivariate polynomials over the field $\bbF$ with
indeterminates given by the variables). We require that the polynomials
are represented as a sum of monomials: that is, every line is of the form
\[
\sum_{s\in\bbN^n}c_s\prod_{i\in\supp(s)}x_i^{s_i}=0
\]
for coefficients $c_s\in\bbF$, where the products $\prod_{i\in\supp(s)}
x_i^{s_i}$ are the {\em monomials} corresponding to the degree vector $s$.
For each variable, the proof system has a {\em Boolean axiom} $[x^2-
x=0]$ (asserting that $x\in\{0,1\}$). The rules of inference are {\em
linear combination}, which asserts that for equations $[p=0]$ and $[q=0]$, for
any coefficients $a$ and $b$ from $\bbF$, we can infer $[a\cdot p+b\cdot q=0]$;
and {\em multiplication}, which asserts that for any variable (indeterminate) 
$x$ and polynomial equation $[p=0]$, we can derive $[x\cdot p=0]$. A 
{\em refutation} in polynomial calculus is a derivation of the polynomial $1$, 
i.e., the contradictory equation $[1=0]$. We will encode ``true'' as the 
equation $[0=0]$, and we will modify the system to allow this equation to be 
asserted as an axiom; of course, it can be derived in a single step from any 
polynomial calculus formula $[p=0]$ by the linear combination $p+(-1)p$, so we 
are essentially not changing the power of the proof system at all.

We also note that without loss of generality, we can restrict our attention to
formulas in which no indeterminate appears in a monomial with degree greater
than one---such monomials are referred to as {\em multilinear}. Intuitively this
is so because the Boolean axioms assert that a larger power can be replaced by a
smaller one; formally, one could derive this as follows: Suppose we have a
formula with a monomial expression $x^k\cdot m$. Then by multiplying the
Boolean axiom by $x$ $k-2$ times, and then by the indeterminates in $m$,
one obtains $[x^k\cdot m-x^{k-1}\cdot m=0]$. A linear combination with
the original formula then yields an expression with the original monomial
replaced by $x^{k-1}\cdot m$, so by repeating this trick $k-2$ additional
times, we eventually reduce the monomial to $x\cdot m$. The same trick can
be applied to the rest of the indeterminates appearing in $m$, and then to the
rest of the monomials in the formula. We will refer to this as the {\em
multilinearization} of the formula. (The original formula could be re-derived
by a similar series of steps, so nothing is lost in this translation.) Looking
ahead, we will be focusing on the {\em degree-bounded} restriction of polynomial
calculus, and so we will assume for simplicity that all formulas are expressed
in this multilinearized (minimal-degree) form. Of course, because the
translation can be performed in a number of steps that is quadratic in the total
degree and linear in the size of the formula, this does not alter the power of
the proof system by much at all.

\paragraph{A note on witnessing and restrictions.}
The polynomial equations can be fit into our framework of restrictions and
witnessing somewhat naturally, thanks to our restriction to the sum of monomials
representation: since we have restricted our attention to cases where each
variable (hence, indeterminate in the polynomial) takes only Boolean values, 
we observe that a monomial corresponds (precisely) to a conjunction over
the set of variables in the support of its degree vector. Then, if say
$\bbF$ is $\bbQ$, we can then express the polynomial equation
\[
\sum_{s\in\bbN^n}c_s\prod_{i\in\supp(s)}x_i^{s_i}=0
\]
in the threshold basis as a conjunction of two thresholds:
\[
\left(\sum_{S\subseteq \{x_1,\ldots,x_n\},S\neq\emptyset}c_S
\bigwedge_{i\in S}x_i\geq -c_\emptyset\right)
\wedge
\left(\sum_{S\subseteq \{x_1,\ldots,x_n\},S\neq\emptyset}-c_S
\bigwedge_{i\in S}x_i\geq c_\emptyset\right)
\]
for $c_S=\sum_{s\in\bbN^n:\supp(s)=S}c_s$. The reader may verify that the
effect of a restriction $\rho$ is now
\[
\left.
\left[\sum_{s\in\bbN^n}c_s\prod_{i\in\supp(s)}x_i^{s_i}=0
\right]\right|_\rho=
\left[\sum_{s\in\bbN^n:\rho_i=0\Rightarrow s_i=0}c_s\prod_{i\in\supp(s):
\rho_i\neq 1}x_i^{s_i}=0
\right]
\]
where we thus denote the polynomial arising from applying $\rho$ to $[p=0]$ by
$p|_\rho$.

This has the effect that the polynomial equation is witnessed true if all of the
monomials (with nonzero coefficients) are witnessed, and the equation evaluates
to 0, and witnessed false if enough of the monomials are witnessed so that
regardless of the settings of the rest of the variables, the sum is either too
large or too small to be zero. Once again, this is a weak kind of ``witnessed
evaluation'' that is nevertheless feasible, and saves us from trying to solve a
system of multivariate polynomial equations---which is easily seen to be NP-hard
(NP-complete if we know we are only interested in Boolean solutions).

\paragraph{Polynomial calculus with resolution.}
Although polynomial calculus can encode the literal $\neg x$ as the
polynomial $(1-x)$, the effect of this choice on the encoding of a clause
is undesirable: for example, recalling the correspondence between monomials and
conjunctions, the clause $x_1\vee\cdots\vee x_n$ corresponds to the
polynomial $(1-x_1)\cdots (1-x_n)$ which has an exponential-size (in
$n$) monomial representation, and hence requires an exponential-size polynomial
calculus formula. In the interest of {\em efficiently} simulating resolution in
polynomial calculus, Alekhnovich et al.~\cite{absrw02} introduced the following
extension of polynomial calculus known as {\em polynomial calculus with
resolution (PCR)}: the formulas are extended by introducing for each variable
$x$, a new indeterminate $\bar{x}$, related by the
{\em complementarity axiom} $[x+\bar{x}-1=0]$ (forcing $\bar{x}=
\neg x$). We can thus represent any clause $\ell_1\vee\cdots\vee\ell_k$ as
a polynomial calculus formula using a {\em single} monomial $[(\neg\ell_1)\cdots
(\neg\ell_k)=0]$ by choosing the appropriate indeterminate for each
$\neg\ell_i$. The reader may verify that in such a case, the cut rule is
captured by adding the monomials (with coefficients of $1$) and weakening may be
simulated by (repeated) multiplication.

For the purposes of (partial) evaluation in PCR, our intended semantics for the
$\bar{x}$ formulas is as follows: a partial assignment $\rho$ assigns
$\rho(\bar{x})=*$ whenever $\rho(x)=*$, and otherwise $\rho(\bar{x})=
\neg\rho(x)$.

\begin{proposition}\label{pc-res-closed}
Polynomial calculus and polynomial calculus with resolution are
restriction-closed.
\end{proposition}
\begin{proof}
Let any partial assignment $\rho$ be given. If a proof step asserts a hypothesis
$[p=0]$, then its restriction $[p|_\rho=0]$ can also be asserted from the
restriction of the hypothesis set. The Boolean axiom $[x^2-x=0]$ can
easily be seen to simplify to $[0=0]$ if $\rho$ assigns a value to $x$, and
otherwise $[x^2-x=0]|_\rho=[x^2-x=0]$, so in the latter case
we can simply assert the Boolean axiom for $x$. For polynomial calculus
with resolution, we need to further consider the complementarity axioms, but as
$\alpha$ is witnessed precisely when $\bar{x}$ is witnessed, we again have
that if $\rho(x)\neq*$, then the complementarity axiom simplifies to
$[0=0]$, and otherwise $[x+\bar{x}-1=0]|_\rho=[x+\bar{x}-1=0]$,
so we can simply assert the corresponding complementarity axiom.

Given our inclusion of $[0=0]$ as an axiom, it only remains to show that the
rules of inference are preserved under partial evaluations. If $\varphi$ is
derived by a linear combination of $[p=0]$ and $[q=0]$ (say $\varphi$ is
$[ap+bq=0]$), then given our encoding of $1$ as the formula $[0=0]$, in any
case, $(ap+bq)|_\rho=a(p|_\rho)+b(q|_\rho)$, so $\varphi|_\rho$ follows by the
same linear combination from $[p=0]|_\rho$ and $[q=0]|_\rho$. If $\varphi$ is
derived by multiplication by $x$ from $[p=0]$, if $\rho(x)=0$, then
$\varphi|_\rho=[0=0]$, which is an axiom. Two cases remain: either $\rho(x)
=1$, in which case $\varphi|_\rho=[p=0]|_\rho$ and so $\varphi|_\rho$ follows
trivially; or, $\rho(x)=*$ and so $\varphi|_\rho=[x\cdot
(p|_\rho)=0]$, so $\varphi$ follows from $[p=0]|_\rho$ by multiplication by
$x$.
\end{proof}

\subsubsection{Degree-bounded polynomial calculus}
%% describe degree bound
Given that the monomial representation of polynomials (in contrast to the
clauses we considered in resolution) may be of exponential size in $n$ (the
number of variables), it is natural to wish to consider a
restricted class of formulas in which the representations of formulas are
guaranteed to be of polynomial size. One way to achieve this is to consider
only degree-$d$ polynomials for some fixed constant $d$---then there are only
$\sum_{i=0}^d{n\choose i}=O(n^d)$ (multilinear) monomials, and so (as long as
the coefficients are reasonably small) we have a polynomial-size representation.
We assume that an ordering of the monomials has been fixed (e.g., in the
representation) such that monomials with larger degree are considered ``larger''
in the ordering. We refer to the first monomial in this ordering with a nonzero
coefficient as the {\em leading monomial} in a polynomial. We will refer to the
{\em degree} of a polynomial calculus or PCR proof as the maximum degree of any
polynomial appearing in a formula used in the proof. We observe that width-$w$
resolution can be simulated by degree-$w$ PCR proofs; thus, in a sense,
degree-bounded polynomial calculus is a natural generalization of width-$w$
resolution.

%% cei algorithm
Degree-bounded polynomial calculus in particular was also first studied by Clegg
et al.~\cite{cei96}. The central observation is that the polynomials derivable 
in bounded degree polynomial calculus form a vector space; the decision 
algorithm (given as Algorithm~\ref{cei-alg}) will then simply construct a basis 
for this space and use the basis to check if the query lies within the space.

\begin{algorithm}[f]
\DontPrintSemicolon
\SetKwInOut{Input}{input}\SetKwInOut{Output}{output}

\Input{Degree bound $d$, list of degree-$d$ polynomials in multilinear monomial
representation $p_1,\ldots,p_\ell$, target degree-$d$ polynomial in multilinear
monomial representation, $q$.}
\Output{{\em Accept} if there is a degree-$d$ polynomial calculus (resp. PCR)
derivation of $[q=0]$; {\em Reject} otherwise.}

\Begin{
Initialize $B$ to the empty list.\\
Initialize $S\leftarrow\{p_1,\ldots,p_\ell\}$ ($S$ also contains the
complementarity polynomials $x+\bar{x}-1$ for PCR).\\
\While{$S\neq\emptyset$}{
  Let $p$ be an arbitrary element of $S$ and remove $p$ from $S$\\
  \ForEach{$b\in B$ in decreasing order (while $p\neq 0$)}{
    \If{The leading monomial in $b$ is the leading monomial in $p$}{
      $p\leftarrow$ Gaussian reduction of $p$ by $b$ (i.e., subtract a multiple
      of $b$ so that the leading monomials cancel).
    }
  }
  \If{$p\neq 0$}{
    Insert $p$ into $B$, maintaining the decreasing order of lead monomials.\\
    \If{$p$ has degree at most $d-1$}{
      \ForEach{indeterminate $\alpha$}{
        Add the multilinearization of $\alpha p$ to $S$.
      }
    }
  }
}
\ForEach{$b\in B$ in decreasing order (while $q\neq 0$)}{
  \If{The leading monomial in $b$ is the leading monomial in $q$}{
    $q\leftarrow$ Gaussian reduction of $q$ by $b$
  }
  \If{$q=0$}{
    \Return{{\em Accept}}
  }
}
\Return{ {\em Reject} }
}

\caption{Pseudocode for Decide-deg-$d$-PC/PCR}\label{cei-alg}
\end{algorithm}

%% analysis

\begin{theorem}[Analysis of decision algorithm for degree-$d$ PC/PCR - Theorem
3, \cite{cei96}]\ 
Algorithm~\ref{cei-alg} solves the limited decision problem for degree-$d$
polynomial calculus (resp. PCR). It runs in time $O((n^d+\ell)n^{2d})$ where $n$
is the number of indeterminates (variables for polynomial calculus,
literals for PCR).
\end{theorem}

As the proof appears in the work of Clegg et al.~\cite{cei96}, we refer the
reader there for details. Clegg et al.~\cite{cei96} also give another algorithm 
based on the Gr\"{o}bner basis algorithm that does not compute an entire basis. 
Although their analysis gives a {\em worse} worst-case running time for this 
alternative algorithm, they believe that it may be more practical; the
interested reader should consult the original paper for details.

%% corollary

In any case, we now return to pursuing our main objective, using Algorithm~\ref
{cei-alg} to obtain algorithms for implicit learning from examples in polynomial
calculus and PCR. We first need to know that the degree-$d$ restrictions of
these proof systems are restriction-closed, which turns out to be easily
established:

\begin{proposition}
For both polynomial calculus and PCR, %polynomial calculus with resolution
the sets of proofs of degree $d$ are restriction-closed.
\end{proposition}
\begin{proof}
We noted in Proposition~\ref{pc-res-closed} that the restriction of any
polynomial calculus (resp. PCR) proof is a valid polynomial calculus (resp. PCR)
proof. Let any partial assignment $\rho$ be given; recalling the connection 
between monomials and conjunctions, we note that for any monomial $x_{i_1}\cdots
x_{i_k}$ $k\leq d$ appearing in a formula in a degree-$d$ polynomial
calculus or PCR proof, the restriction under $\rho$ is $0$ (of degree $0$) if
any $x_{i_j}$ is set to $0$ by $\rho$, and otherwise it is $\prod_{j:
\rho(x_{i_j})=*}x_{i_j}$, which has degree at most $k\leq d$. Thus,
the degrees can only decrease, so the restriction of the proof under $\rho$ is
also a degree-$d$ proof.
\end{proof}

We therefore obtain the following corollary from Theorem~\ref
{implicit-learn-thm}:

\begin{corollary}[Implicit learning in degree-bounded polynomial calculus and
PCR]\label{pc-implicit-cor}
Let a list of degree-$d$ polynomials $p_1,\ldots,p_\ell$ and $q$ be given, and
suppose that partial assignments are drawn from a masking process for an 
underlying distribution $D$; suppose further that either
\begin{compactenum}
\item There exists some list of polynomials $h_1,\ldots,h_k$ such that partial
assignments from the masking process are witnessed to satisfy $[h_1=0],\ldots,
[h_k=0]$ with probability at least $(1-\epsilon+\gamma)$ and there is a
degree-$d$ polynomial calculus (resp. PCR) derivation of $[q=0]$ from
$[p_1=0],\ldots,[p_\ell=0],[h_1=0],\ldots,[h_k=0]$
or else
\item $[(p_1=0)\wedge\cdots\wedge(p_\ell=0)\Rightarrow (q=0)]$ is at most
$(1-\epsilon-\gamma)$-valid with respect to $D$ for $\gamma>0$.
\end{compactenum}
Then, there an algorithm running in time
$O(\frac{\ell+n^d}{\gamma^2}n^{2d}\log\frac{1}{\delta})$ (given unit cost
field operations) that distinguishes these cases with probability $1-\delta$
when given $q$, $p_1,\ldots,p_\ell$, $\epsilon$, $\gamma$, and a
sample of $O(\frac{1}{\gamma^2}\log\frac{1}{\delta})$ partial assignments.
\end{corollary}

\subsection{Sparse, bounded cutting planes}
In {\em integer linear programming}, one is interested in determining integer
solutions to a system of linear inequalities; {\em cutting planes}~\cite
{gomory60} were introduced as a technique to improve the formulation of an
integer linear program by deriving new inequalities that are satisfied by
the integer solutions to the system of inequalities, but not by all of the
fractional solutions. The current formulation of cutting planes is due to
Chv\'{a}tal~\cite{chvatal73}, and it was explicitly cast as a propositional
proof system by Cook et al.~\cite{cct87} where the objective is to prove that a
system has {\em no} feasible integer solutions. Much like resolution, cutting
planes are not only simple and natural, surprisingly, they are also
complete~\cite{chvatal73,cct87}. Furthermore, Cook et al.~\cite{cct87} noted
that cutting planes could easily simulate resolution, and that some formulas
that were hard for resolution (encoding the ``pigeonhole principle'') had simple
cutting plane proofs.

We can also give a syntactic analogue of bounded-width in resolution for cutting
planes which will enable us to state a limited decision problem with an
efficient algorithm. Although this restriction of cutting planes will not be
able to express the hard examples for resolution, their simplicity and
connections to optimization make them a potentially appealing direction for
future work. 
%formulas for which algorithms such as WINNOW~\cite{littlestone88} achieve
%a logarithmically small sample complexity.

\paragraph{The cutting planes proof system.}
The formulas of cutting planes are inequalities of the form
$[\sum_{i=1}^kc_ix_i\geq b]$ where each $x_i$ is a variable and
$c_1,\ldots,c_k$ and $b$ are integers. Naturally, we will restrict our attention
to $\{0,1\}$-integer linear programs (i.e., Boolean-valued), so our system will
feature axioms of the form $x\geq 0$ and $-x\geq -1$ (i.e.,
$x\leq 1$) for each variable $x$. Naturally, we will allow the
{\em addition} of two linear inequalities: given $\varphi^{(1)}=[\sum_{i=1}^k
c^{(1)}_ix_i\geq b^{(1)}]$ and $\varphi^{(2)}=[\sum_{i=1}^kc^{(2)}_i
x_i\geq b^{(2)}]$, we can derive $\varphi^{(1)}+\varphi^{(2)}=[\sum_{i=1}^k
(c^{(1)}_i+c^{(2)}_i)x_i\geq b^{(1)}+b^{(2)}]$. We will also allow
ourselves to {\em multiply} an inequality $[\sum_{i=1}^kc_ix_i\geq
b]$ by any positive integer $d$ to obtain $[\sum_{i=1}^k(d\cdot c_i)x_i
\geq d\cdot b]$. Finally, the key rule is {\em division}: given an inequality
of the form $[\sum_{i=1}^k(d\cdot c_i)x_i\geq b]$ for a positive integer
$d$ (i.e., a common divisor of the coefficients) we can derive $[\sum_{i=1}^kc_i
x_i\geq \lceil b/d\rceil ]$; crucially, this derivation is only sound due
to the fact that the $x_i$ are assumed to take {\em integer} values. It is
the fact that this rounding may ``cut'' into the region defined by the system
of linear inequalities that gives the proof system its name. A {\em refutation}
in cutting planes is a derivation of the (contradictory) inequality $[0\geq 1]$.

Again, we will need to make some technical modifications that do not change the
power of the proof system by much. We will encode $1$ as an axiom by the
inequality $0\geq -1$ which, we note, can be trivially derived in two steps by
the standard formulation of cutting planes. We will also introduce a
{\em weakening} rule: consider any linear inequality $[\sum_{i=1}^kc^{(1)}_i
x_i\geq b^{(1)}]$ that is witnessed true in every partial assignment,
specifically in the one that masks all variables---this means that
$\sum_{i=1}^k\min\{0,c^{(1)}_i\}\geq b$. Then, from any linear inequality
$[\sum_{i=1}^kc^{(2)}_ix_i\geq b^{(2)}]$, we will allow ourselves to derive
$[\sum_{i=1}^k(c^{(1)}_i+c^{(2)}_i)x_i\geq b^{(1)}+b^{(2)}]$ in a single
step. Of course, $[\sum_{i=1}^kc^{(1)}_ix_i\geq b^{(1)}]$ could be derived
from the axioms in at most $3n+2$ steps if there are $n$ variables while
using only two formulas' worth of space, whereupon the final inequality follows
by addition.

We will also find the following observation convenient: as restrictions are a 
kind of partial evaluation, it is intuitively clear that we can perform the 
evaluation in stages and obtain the same end result, that is:

\begin{proposition}[Restrictions may be broken into stages]\label{staged-eval}
Let $\rho$ be a partial assignment, and let $\sigma$ be another partial 
assignment such that for every variable $x_i$, whenever $\rho_i=*$, $\sigma_i=
*$, and whenever $\sigma_i\in\{0,1\}$, $\sigma_i=\rho_i$. Now, let $\tau$ be a
partial assignment to the variables $\{x_i:\sigma_i=*\}$ such that for every 
$x_i$, $\sigma_i=\rho_i$. Then for every formula $\varphi$, $\varphi|_\rho=
(\varphi|_\sigma)|_\tau$.
\end{proposition}
\begin{proof}
We can verify this by induction on the construction of $\varphi$:
\begin{itemize}
\item Naturally, for variables $x_i$, either $\rho_i=*$, in which case
$x_i|_\rho=x_i=(x_i|_\sigma)|_\tau$, or else $\rho_i\in\{0,1\}$
in which case either $\sigma_i=\rho_i$, or else $x_i|_\sigma=x_i$, and then 
$\tau_i=\rho_i$.
\item If $\varphi=\neg\psi$, we have by the induction hypothesis that
$\psi|_\rho=(\psi|_\sigma)|_\tau$. Regardless of whether or not $\varphi$ is
%When $\psi$ is witnessed in $\rho$, $\varphi|_\rho=\neg(\psi|_\rho)=
%(\varphi|_\sigma)|_\tau$, whereas when $\psi$ is not
witnessed, $\varphi|_\rho=\neg(\psi|_\rho)=\neg((\psi|_\sigma)|_\tau)=
(\varphi|_\sigma)|_\tau$.
\item If $\varphi=[\sum_{i=1}^kc_i\psi_i\geq b]$, we again have by the induction
hypothesis that for every $\psi_i$, $\psi_i|_\rho=(\psi_i|_\sigma)|_\tau$, and
thus, the same $\psi_i$ are witnessed (to evaluate to true or false) in both
cases.
\begin{itemize}
\item If $\varphi$ is not witnessed in $\rho$, it is then immediate that
$\varphi|_\rho=(\varphi|_\sigma)|_\tau$.
\item If $\varphi$ is witnessed in $\rho$, but not witnessed in $\sigma$, we
observe that $\varphi$ must be witnessed in $\tau$ since the same set of
formulas are witnessed to evaluate to true and false in both cases, and
therefore also again, $\varphi|_\rho=(\varphi|_\sigma)|_\tau$.
\item Finally, when $\varphi$ is witnessed in $\sigma$, we note that by the
construction of witnessed values, it does not matter what values the formulas
witnessed by $\rho$ but not $\sigma$ take---$\varphi$ must be witnessed to take
the same value under both $\rho$ and $\sigma$. Then since
$(\varphi|_\sigma)|_\tau=\varphi|_\sigma\in\{0,1\}$, we see once again
$(\varphi|_\sigma)|_\tau=\varphi|_\rho$.
\end{itemize}
\end{itemize}
\end{proof}

\begin{proposition}\label{cp-res-closed}
Cutting planes is restriction-closed.
\end{proposition}
\begin{proof}
We are again given that our encoding of $1$, $0\geq -1$, is an axiom. Now, let
any partial assignment $\rho$ be given. Again, for any hypothesis $\varphi$,
asserted in the proof, $\varphi|_\rho$ can be asserted from the set of
restrictions of hypotheses. Likewise, for each axiom, if $\rho$ assigns the
variable a value, then it simplifies to $1$ (which is given as an axiom
by assumption) and otherwise it remains an assertion of the same axiom, so in
either case it may still be asserted as an axiom. It thus remains to consider
formulas derived by our four inference rules.

We thus consider any formula $\varphi$ derived in the proof that is not
witnessed to evaluate to true in $\rho$. If it was derived from a formula $\psi$
by weakening, we note that if $\psi|_\rho=1$ (i.e., was witnessed to evaluate to
true), then since $\varphi$ is the sum of $\psi$ and another inequality $\xi$
that is witnessed to evaluate to true, we would have $\varphi|_\rho=1$ also, but
it is not by assumption. Therefore also $\psi|_\rho\neq 1$. Furthermore, by
Proposition~\ref{staged-eval}, $\xi|_\rho$ is (also) witnessed true on every
further partial assignment. Therefore,
$\varphi|_\rho=(\psi+\xi)|_\rho$ follows from $\psi|_\rho$ by weakening (with
$\xi|_\rho$). Similarly, if $\varphi$ was derived by addition of
$\psi$ and $\xi$, at least one of $\psi$ and $\xi$ must not be witnessed to
evaluate to $1$ under $\rho$; WLOG suppose it is $\psi$. Then if $\xi|_\rho=1$,
$\varphi|_\rho$ again follows from $\psi|_\rho$ by weakening. Finally, if
neither $\psi$ nor $\xi$ is witnessed to evaluate to true under $\rho$,
we can derive $\varphi|_\rho$ from $\psi|_\rho$ and $\xi|_\rho$ by addition.

Multiplication is especially simple: we note that if $\varphi=[\sum_{i=1}^k(d
\cdot c_i)x_i\geq d\cdot b]$ is derived by multiplication from $\psi=
[\sum_{i=1}^kc_ix_i\geq b]$, then $\psi$ also follows from $\varphi$ by
division, and hence $\varphi|_\rho=1$ iff $\psi|_\rho=1$ in this case; as we
have assumed $\varphi|_\rho\neq 1$, we note that we can derive $\varphi|_\rho$
from $\psi|_\rho$ by multiplication by the same $d$. Finally, if $\varphi=
[\sum_{i=1}^kc_ix_i\geq \lceil b/d\rceil ]$ was derived from $\psi=
[\sum_{i=1}^k(d\cdot c_i)x_i\geq b]$ by division, we note (more carefully)
that if $\psi|_\rho=1$, then as this means that $\sum_{i:\rho_i=1}
\min\{0,d\cdot c_i\} \geq b$ where the LHS is an integer, and hence also
$\sum_{i:\rho_i=1}\min\{0, c_i\} \geq \lceil b/d\rceil$, so $\varphi$
would also be witnessed to evaluate to true, but we have assumed it does not.
Now, we note that
\[
\psi|_\rho = \left[\sum_{i:\rho_i=*}(d\cdot c_i)x_i \geq
\left(b-\sum_{i:\rho_i=1}(d\cdot c_i)\right)\right]
\]
where division by $d$ therefore yields
\[
\left[\sum_{i:\rho_i=*}c_ix_i \geq
\left\lceil\frac{b}{d}-\sum_{i:\rho_i=1}c_i\right\rceil\right]
=\left[\sum_{i:\rho_i=*}c_ix_i \geq
\left(\left\lceil\frac{b}{d}\right\rceil-\sum_{i:\rho_i=1}c_i\right)
\right]=\varphi|_\rho
\]
as $\sum_{i:\rho_i=1}c_i$ is an integer.
\end{proof}

\subsubsection{Efficient algorithms for sparse, $\ell_1$-bounded cutting planes}

We now turn to developing a syntactic restriction of cutting planes that
features an efficient limited decision algorithm.

\paragraph{Sparse cutting planes.} The main restriction we use is to limit the
number of variables appearing in the threshold expression: we say that the
formula is {\em $w$-sparse} if at most $w$ variables appear in the
sum.\footnote{%
Naturally, this is a direct analogue of width in resolution; the reason we do
not refer to it as ``width'' is that in the geometric setting of cutting planes,
width strongly suggests a geometric interpretation that would be inappropriate.}
Naturally, we say that a cutting planes proof is $w$-sparse if every formula
appearing in the proof is $w$-sparse.

\paragraph{$\ell_1$-bounded coefficients.} We will also use a restriction on the
magnitude of the (integer) coefficients. Given a formula of cutting planes,
$\varphi=[\sum_{i=1}^kc_ix_i\geq b]$, we define the {\em $\ell_1$-norm of
$\varphi$} (denoted $\|\varphi\|_1$) to be $|b|+\sum_{i=1}^k|c_i|$, i.e., the
$\ell_1$ norm of the coefficient vector. For $L\in\bbN$, we naturally say that a
cutting planes proof is $L$-bounded if every $\varphi$ appearing in the proof
has $\|\varphi\|_1\leq L$.

We remark that the natural simulation of width-$w$ resolution by cutting planes
yields $w$-sparse and $2w$-bounded proofs: intuitively, we wish to encode a
clause $C=\ell_1\vee\cdots\vee\ell_k$ by the linear inequality
\[
\sum_{i:\ell_i=x_j}x_j+\sum_{i:\ell_i=\neg x_j}(1-x_j)\geq 1
\]
which naturally corresponds to the cutting planes formula
\[
\left[\sum_{i:\ell_i=x_j}x_j+\sum_{i:\ell_i=\neg x_j}(-1)x_j
\geq 1-|\{i:\ell_i\mathrm{\ negative}\}|\right]
\]
in which, if $k\leq w$, the coefficients from the LHS contribute at most $w$
to the $\ell_1$-norm, and the threshold is easily seen to contribute at most
$w$ (assuming $w\geq 1$). So, a simultaneously sparse and $\ell_1$-bounded
restriction of cutting planes generalizes the width-bounded restriction of
resolution.

We furthermore need to know that this special case of cutting planes is
restriction-closed---note that other natural special cases, e.g., bounding the
sizes of {\em individual} coefficients may not be. Nevertheless, for the
$\ell_1$-bounded cutting planes, this is easily established:

\begin{proposition}
The class of $L$-bounded $w$-sparse cutting plane proofs is restriction closed
for any $L,w\in\bbN$.
\end{proposition}
\begin{proof}
Let any $L$-bounded $w$-sparse cutting plane proof $\Pi$ and partial assignment
$\rho$ be given. We consider the proof $\Pi|_\rho$ obtained by restricting every
step of $\Pi$ by $\rho$ (shown to be a cutting planes proof in Proposition~\ref
{cp-res-closed}). Now, we note that in this proof, our encoding of $1$ as
$[0\geq -1]$ is $0$-sparse and $1$-bounded, so it is guaranteed to be
$L$-bounded and $w$-sparse. More generally, given any $\varphi$ that is
$L$-bounded and $w$-sparse,
\[
\varphi|_\rho=\left[\sum_{i:\rho_i=*}c_ix_i
\geq b-\sum_{i:\rho(\alpha_i)=1}c_i\right]
\]
has $\ell_1$-norm
\[
\|\varphi|_\rho\|_1=\left|b-\sum_{i:\rho_i=1}c_i\right|+\sum_{i:
\rho_i=*}|c_i|\leq |b|+\sum_{i:\rho_i\neq 0}|c_i|
\]
by the triangle inequality; as furthermore $0\leq\sum_{i:\rho_i=0}
|c_i|$, we conclude that $\|\varphi|_\rho\|_1\leq\|\varphi\|_1\leq L$, so
$\Pi|_\rho$ is also $L$-bounded. Similarly, since every variable appearing
in $\varphi|_\rho$ appears in $\varphi$ and $\varphi$ appearing in $\Pi$ are
assumed to be $w$-sparse, $\varphi|_\rho$ appearing in $\Pi|_\rho$ are also
$w$-sparse. Thus, $\Pi|_\rho$ is also a $w$-sparse cutting planes proof, as
needed.
\end{proof}

%% describe the DP SBCP algorithm

\begin{algorithm}[t]
\DontPrintSemicolon
\SetKwInOut{Input}{input}\SetKwInOut{Output}{output}

\Input{Formulas $\varphi_1,\ldots,\varphi_\ell$ and
$\phi$, sparsity and $\ell_1$-norm bounds $w,L\in\bbN$.}
\Output{{\em Accept} if there is a $L$-bounded $w$-sparse proof
of $\phi$ of from $\varphi_1,\ldots,\varphi_\ell$; else, {\em Reject}.}

\Begin{
Initialize a table $T[\psi]\leftarrow 0$ for every cutting planes formula $\psi$
of sparsity $w$ and $\|\psi\|_1\leq L$; put $T[\psi]\leftarrow 1$ for every
axiom $\psi$.
\If{$\phi$ an axiom}{\Return{\em Accept}}
 \For{$i=1,\ldots,\ell$ if $\varphi_i$ is $w$-sparse}{
   \If{$\varphi_i=\psi$}{\Return {\em Accept}}
   $T[\varphi_i]\leftarrow 1$
 }
$NEW\leftarrow 1$.\\
\While{$NEW=1$}{
 $NEW\leftarrow 0$.\\
 \ForEach{Pair of formulas $(\psi_1,\psi_2)$ in $T$ or among $\varphi_1,\ldots,
\varphi_\ell$}{
    \If{$\psi_1+\psi_2$ has sparsity at most $w$, $\|\psi_1+\psi_1\|\leq L$, and
$T[\psi_1+\psi_2]=0$}{
        $NEW\leftarrow 1$; $T[\psi_1+\psi_2]\leftarrow 1$
       }
     }
  \ForEach{Formula $\psi$ in $T$}{
    \For{$a=-L,\ldots,L$}{
      \If{$\|a\cdot\psi\|_1\leq L$ and $T[a\cdot\psi]=0$}{
        \If{$a\cdot\psi=\phi$}{\Return {\em Accept}}
        $NEW\leftarrow 1$; $T[a\cdot\psi]\leftarrow 1$
        }
    }
    \For{$d=2,\ldots,L$}{
      \If{$d$ divides $\psi$ and $T[\psi\mathrm{\ divided\ by\ }d]=0$}{
        \If{$\psi$ divided by $d=\phi$}{\Return {\em Accept}}
        $NEW\leftarrow 1$; $T[\psi\mathrm{\ divided\ by\ }d]\leftarrow 1$\\
       }
     }
   }
 }
\Return{ {\em Reject} }

}

\caption{DecideSparseBoundedCP}\label{cp-decision-alg}
\end{algorithm}

We now consider Algorithm~\ref{cp-decision-alg}, an analogue of Algorithm~\ref
{res-k-width-decision-alg} -- i.e., a simple dynamic programming algorithm -- 
for the limited decision problem for $w$-sparse and $L$-bounded cutting planes.

%% analysis of DSBCP

\begin{theorem}[Analysis of decision algorithm for sparse, bounded cutting
planes]\label{sbcp-analysis}
For any $w,L\in\bbN$, Algorithm~\ref{cp-decision-alg} solves the limited
decision problem for $w$-sparse $L$-bounded cutting planes. It runs in time
$O((w+\max\{|\phi_i|\})L(Ln)^w(L(Ln)^w+\ell)^2)$ (which, for $w$
constant and $w$-sparse $L$-bounded $\phi_i$ is $O(L^3(Ln)^{3w})$)
where $n$ is the number of variables.
\end{theorem}
\begin{proof}
The analysis is very similar to our previous dynamic programming algorithms for
bounded-width RES$(k)$, Theorem~\ref{res-k-width-analysis}. As there, we are 
inductively guaranteed that at each stage we set $T[\psi]$ to $1$ only when 
there is a $w$-sparse $L$-bounded proof of $\psi$, and conversely, for every 
$\psi$ with a $w$-sparse $L$-bounded proof, until $T[\psi]$ is set to $1$, on 
each iteration of the main loop, we set an entry of $T$ to $1$ for some new step
of the proof (we noted that weakening could be simulated by repeated addition of
axioms, so we don't need to consider it explicitly). Thus, if the input target 
$\phi$ has a $w$-sparse $L$-bounded proof, $T[\phi]$ would be set to $1$ at some
point, whereupon the algorithm accepts, and otherwise since the size of the 
table is bounded, the algorithm eventually cannot add more formulas to the table
and so rejects. It only remains to consider the running time.

The main observation is that there are at most ${w+1+L\choose w+1}=O(L^{w+1})$
ways of assigning integer weights of total $\ell_1$-weight at most $L$ to the
$w$ nonzero coefficients and the threshold; therefore, as there are at most
$O(n^w)$ distinct choices of up to $w$ variables, there are at most
$O(L^{w+1}n^w)$ possible $w$-sparse $L$-bounded cutting plane formulas. At least
one is added on each iteration of the loop, and each iteration considers every
pair of such formulas with the $\ell$ input formulas (for $O((L(Ln)^w+\ell)^2)$
pairs on each iteration), where this sum can be carried out and checked in
$O(w+\max\{|\phi_i|\})$ arithmetic operations; checking the $O(L)$ possible
multiples and divisors for each of the $O(L(Ln)^w)$ formulas in $T$ also takes
$O(w)$ arithmetic operations each, so the time for adding pairs dominates. The
claimed running time is now immediate.
\end{proof}

%% the corollary
Once again, we are in a position to apply Theorem~\ref{implicit-learn-thm}, and
thus obtain:

\begin{corollary}[Implicit learning in sparse bounded cutting planes]
\label{cp-implicit-cor}
Let a list of $w$-sparse $L$-bounded cutting planes formulas $\varphi_1,\ldots,
\varphi_\ell$ and $\phi$ be given, and suppose that 
partial assignments are drawn from a masking process for an underlying
distribution $D$; suppose further that either
\begin{compactenum}
\item There exists some list of cutting planes formulas $\psi_1,\ldots,\psi_k$
such that partial assignments from the masking process are witnessed to satisfy
$\psi_1,\ldots,\psi_k$ with probability at least $(1-\epsilon+\gamma)$ and there
is a $w$-sparse $L$-bounded cutting planes derivation of $\phi$ from
$\varphi_1,\ldots,\varphi_\ell,\psi_1,\ldots,\psi_k$
or else
\item $[\varphi_1\wedge\cdots\wedge\varphi_\ell\Rightarrow\phi]$ is at most
$(1-\epsilon-\gamma)$-valid with respect to $D$ for $\gamma>0$.
\end{compactenum}
Then, there an algorithm running in time $O(\frac{w+\max\{|\phi_i|\}}{\gamma^2}
L(Ln)^w(L(Ln)^w+\ell)^2\log\frac{1}{\delta})$ (given unit cost arithmetic
operations) that distinguishes these cases with probability $1-\delta$ when
given $\phi$, $\varphi_1,\ldots,\varphi_\ell$, $\epsilon$, $\gamma$, and a
sample of $O(\frac{1}{\gamma^2}\log\frac{1}{\delta})$ partial assignments.
\end{corollary}

\section{The utility of knowledge with imperfect validity}
\label{imperfect-rules}

Although our introduction of PAC-Semantics was primarily motivated by our need
for a weaker guarantee that could be feasibly satisfied by inductive learning
algorithms, it turns out to provide a windfall from the standpoint of several
other classic issues in artificial intelligence. Several such examples are
discussed by Valiant~\cite{valiant95};\footnote{%
Concerning a related, but slightly different framework---there, ``unspecified''
is taken to be a third value, on par with ``true'' and ``false,'' which may be
treated specially in reasoning.} we will dwell on two core, related problems
here, the {\em frame} and {\em qualification} problems, first discussed by
McCarthy and Hayes~\cite{mcch69}. The frame problem essentially concerns the
efficient representation of what changes -- and what {\em doesn't} -- as the
result of an action (stressed in this form by Raphael~\cite{raphael71}). The
traditional solutions to this problem -- first suggested by Sandewall~\cite
{sandewall72}, with a variety of subsequent formalizations including notably,
McCarthy's circumscription~\cite{mcc80,mcc86} and Reiter's defaults~\cite
{reiter80} and ``successor state axioms''~\cite{reiter91} -- all essentially are
(informally) captured by asserting in one way or another that (normally)
``nothing changes unless an action that changes it is taken.'' Putting the early
methods such as circumscription and defaults aside (which have their own issues,
cf. Hanks and McDermott's ``Yale shooting problem''~\cite{hmd87}), the other
approaches make the above assertion explicit, and thus encounter some form of
the qualification problem---that is, it is essentially impossible to assert the
full variety of reasons for and ways in which something could change or fail to
change in a real-world situation.

Thus, the successor state axioms (etc.) fully capture a toy domain at best. And
yet, such simplified models have shown to be useful in the design of algorithms
for planning---implicitly in early work such as Fikes and Nilsson's
STRIPS~\cite{fn71}, and more explicitly in later work such as Chapman's ``modal
truth criterion'' in his work on partial-order planning~\cite{chapman87} and
as explicit constraints in planning as propositional satisfiability by Kautz and
Selman~\cite{ks92,ks96}. Indeed, such approaches ``solve the problem'' in the
sense that the kinds of plans generated by such systems are intuitively
reasonable and correspond to what is desired.

More to the point, we can take the stance that such assumptions are merely
{\em approximations} to the real-world situation that may fail for various
unanticipated reasons, and so while the plans generated on their basis may
likewise fail for unanticipated reasons, this does not detract from the utility
of the plans under ordinary circumstances. Indeed, supposing we take a
discrete-time probabilistic (e.g., Markovian) model of the evolution of the
world, we might reasonably expect that if we consider the marginal distribution
over successive world states, that formulas such as the successor state axioms
would be $(1-\epsilon)$-valid with respect to this distribution for some small
(but nonzero) $\epsilon$. Of course, this view of the solutions to the frame 
problem is {\em not} novel to this work, and it has been expressed since the 
earliest works on probabilistic models in planning~\cite{dk89,gd94}. The point 
is rather that such examples of what are effectively $(1-\epsilon)$-valid rules 
arise naturally in applications, and we claim that just as PAC-Semantics 
captures the sense in which learned rules are (approximately) ``true,'' 
PAC-Semantics also captures the sense in which these approximate rules (e.g., as
used in planning) are ``true.''

\section{Directions for future work}

A broad possible direction for future work involves the development of 
algorithms for reasoning in PAC-Semantics {\em directly}, that is, not obtained 
by applying Theorem~\ref{implicit-learn-thm} to algorithms for the limited 
decision problems under the classical (worst-case) semantics of the proof 
systems. We will give some concrete suggestions for how this might be pursued
below.

\subsection{Incorporating explicit learning}

One approach concerns the architecture of modern algorithms for deciding
satisfiability; a well-known result due to Beame et al.~\cite{bks04} establishes
that these algorithms effectively perform a search for resolution proofs of
unsatisfiability (or, satisfying assignments), and work by Atserias et al.~\cite
{aft11} shows that these algorithms (when they make certain choices at random)
are effective for deciding bounded-width resolution.

The overall architecture of these modern ``SAT-solvers'' largely follows that
of Zhang et al.~\cite{zmmm01}, and is based on improvements to DPLL~\cite{dp60,
dll62} explored earlier in several other works~\cite{mss99,bs97,gsc97}. Roughly
speaking, the algorithm makes an arbitrary assignment to an unassigned variable,
and then examines what other variables must be set in order to satisfy the
formula; when a contradiction is entailed by the algorithm's decision, a new
clause is added to the formula (entailed by the existing clauses) and the
search continues on a different setting of the variables. A few simple rules are
used for the task of exploring the consequences of a partial setting of the
variables---notably, for example, {\em unit propagation}: whenever all of the 
literals in a clause are set to false except for one (unset) variable, that 
final remaining literal must be set to true if the assignment is to satisfy the 
formula. 

One possibility for improving the power of such algorithms for reasoning under
PAC-Semantics using examples is that one might wish to use an explicit learning 
algorithm such as WINNOW~\cite{littlestone88} to learn additional (approximately
valid) rules for extending partial assignments. If we are using these algorithms
to find resolution refutations, then when a refutation was produced by such a
modified architecture, it would establish that the input formula is only
satisfied with some low probability (depending on the error of the learned rules
that were actually invoked during the algorithm's run).

Given such a modification, one must then ask: does it actually improve the power
of such algorithms? Work by Pipatsrisawat and Darwiche~\cite{pd11} (related to
the above work) has shown that with appropriate (nondeterministic) guidance in 
the algorithm's decisions, such algorithms do actually find arbitrary (i.e., 
DAG-like) resolution proofs in a polynomial number of iterations. Yet, it is
still not known whether or not a feasible decision strategy can match this.
Nevertheless, their work (together with the work of Atserias et al.~\cite
{aft11}) provides a potential starting point for such an analysis.

\subsubsection{A suggestion for empirical work}\label{empirical-future-work}

Another obvious direction for future work is the development and tuning of
real systems for inference in PAC-Semantics. While the algorithms we have 
presented here illustrate that such inference can be theoretically rather 
efficient and are evocative of how one might approach the design of a real-world
algorithm, the fact is that (1) any off-the-shelf SAT solver can be easily 
modified to serve this purpose and (2) SAT solvers have been highly optimized by
years of effort. It would be far easier and more sensible for a group with an 
existing SAT solver implementation to simply make the following modification, 
and see what the results are: along the lines of Algorithm~\ref{search-space}, 
for a sample of partial assignments $\{\rho^1,\ldots,\rho^m\}$, the algorithm 
loops over $i=1,\ldots,m$, taking the unmasked variables in $\rho^i$ as 
decisions and checks for satisfiability with respect to the remaining variables.
Counting the fraction of the partial assignments that can be extended to 
satisfying assignments then gives a bound on the validity of the input formula. 
Crucially, in this approach, {\em learned clauses are shared across samples}. 
Given that there is a common resolution proof across instances (cf. the 
connection between SAT solvers and resolution~\cite{bks04}) we would expect this
sharing to lead to a faster running time than simply running the SAT solver as a
black box on the formulas obtained by ``plugging in'' the partial assignments 
(although that is another approach).

\subsection{Exploiting limited kinds of masking processes}

Another direction for possibly making more sophisticated use of the examples in
reasoning under PAC-Semantics involves restricting the masking processes. In the
pursuit of reasoning algorithms, it might be helpful to consider restrictions 
that allow some possibility of ``extrapolating'' from the values of variables 
seen on one example to the values of hidden variables in other examples (which 
is not possible in general since the masking process is allowed to ``see'' the 
example before choosing which entries to mask). For example, if the masks were 
chosen independently of the underlying examples, this might enable such guessing
to be useful.

\subsection{Relating implicit learning to query-driven explicit learning}

A final question that is raised by this work is whether or not it might be 
possible to extend the algorithm used in Theorem~\ref{implicit-learn-thm}, 
Algorithm~\ref{pac-decision-alg}, to produce an explicit proof from an explicit 
set of formulas that are satisfied with high probability from e.g., algorithms 
for finding treelike resolution proofs even when the CNF we need is not 
perfectly valid. Although this is a somewhat ambitious goal, if one takes 
Algorithm~\ref{pac-decision-alg} as a starting point, the problem is of a 
similar form to one considered by Dvir et al.~\cite{drwy12}---there, they 
considered learning decision trees from restrictions of the target tree. The 
main catch here is that in contrast to their setting, we are not guaranteed that
we find restrictions of the same underlying proof, even when one is assumed to 
exist.

\section*{Acknowledgements}
This work was heavily influenced by conversations with Leslie Valiant. 
\appendix
\section*{Appendix}

\section{The necessity of computationally feasible witnessing}
\label{feasible-witness}
We now show that it is necessary for our implicit learning problem that any
notion of witnessing we use possess some kind of efficient algorithm.
Broadly speaking, we are supposing that we use some class of ``axiom'' formulas
$A$ such that whenever the collection of axioms $\{\alpha_1,\ldots,\alpha_k\}
\subseteq A$ satisfy the our candidate witnessing property $W$ (given as a
relation over, say, formulas and partial assignments) under the masking process 
with probability $(1-\epsilon)$ (guaranteeing that $\alpha_1\wedge\cdots\wedge
\alpha_k$ is $(1-\epsilon)$-valid for the underlying distribution $D$), and
there exists a proof $\Pi$ of the query $\varphi$ in the limited set $\calS$
from the set of hypotheses $\{\alpha_1,\ldots,\alpha_k\}$, then the algorithm
certifies the $(1-\epsilon)$-validity of the query $\varphi$ under $D$. Now, in
general, we would expect that in any ``reasonable'' proof system and class of
``simple'' proofs $\calS$, the hypotheses should have trivial proofs (namely,
they can be asserted immediately) and therefore the efficient algorithm we are
seeking should certify the $(1-\epsilon)$-validity of any member of $A$ whenever
the property $W$ holds for the masking process with probability $(1-\epsilon)$.
(We will repeat this argument slightly more formally in Proposition~\ref
{axioms-checkable} below.)

In summary, this means precisely that for such a collection $A$, there is an
algorithm such that on input $\alpha\in A$ (and $\delta,\gamma>0$) and
given an oracle for examples, for any distribution over masked examples given
by a masking process applied to a distribution over scenes $M(D)$, with
probability at least $1-\delta$ the algorithm correctly decides whether
$\Pr_{\rho\in M(D)}[W(\alpha,\rho)]\geq 1-\epsilon+\gamma$ or
$\Pr_{x\in D}[\alpha(x)=0]\geq \epsilon+\gamma$ (given that one of these cases
holds) in time polynomial in the size of the domain, $1/\gamma$,
$\log 1/\delta$, $\log 1/\epsilon$, and the size of $\alpha$. We refer to this
algorithm as an {\em efficient PAC-Certification of $W$ for $A$}, and it serves
as a kind of efficient evaluation algorithm for $W$.

We now restate these observations more formally: any notion of ``witnessing''
underlying an implicit learning algorithm in the style of Theorem~\ref
{implicit-learn-thm} must be efficiently evaluable on partial assignments and
therefore also verifiable from examples.

\begin{proposition}[Witnessing of axioms must be computationally feasible]
\label{axioms-checkable}
Let $\calS$ be a set of proofs for a proof system such that any explicit
hypothesis has a proof in $\calS$. Let $A$ be a set of formulas and $W$ be a
property of formulas.

Suppose that there is a probabilistic algorithm running in time polynomial in
the number of variables $n$, the size of the query and set of hypotheses,
$1/\gamma$, and the number of bits of precision of the parameters $\epsilon$ and
$\delta$ with the following behavior: given a query formula $\varphi$,
$\epsilon,\delta,\gamma\in (0,1)$, query access to example partial assignments 
from a masking process $M$ over a distribution over assignments $D$, and a list 
of hypothesis formulas $H$, distinguishes
\begin{itemize}
\item queries $\varphi$ such that $[H\Rightarrow\varphi]$ is not
$(1-\epsilon-\gamma)$-valid under $D$ from
\item queries that have a proof in $\calS$ from $H'=H\cup A'$ for some $A'
\subseteq A$ such that $$\Pr_{\rho\in M(D)}[\forall \alpha\in A'\ W(\alpha,
\rho)]\geq 1-\epsilon+\gamma.$$
\end{itemize}
Then there is a probabilistic polynomial time algorithm that on input $\alpha\in
A$ and $\rho$ distinguishes pairs for which $W$ holds from pairs for which there
is some $x$ consistent with $\rho$ such that $\alpha(x)=0$.

Moreover, for $\{\alpha_1,\ldots,\alpha_k\}$ and an oracle for examples
from some distribution over partial assignments $M(D)$, we can distinguish
\[
\Pr_{\rho\in M(D)}[W(\alpha_1,\rho)\wedge\cdots\wedge W(\alpha_k,\rho)]\geq
1-\epsilon+\gamma
\]
from cases where $\alpha_1\wedge\cdots\wedge\alpha_k$ is not $(1-\epsilon-
\gamma)$-valid with probability $1-\delta$ in time polynomial in $1/\gamma$,
$\log 1/\epsilon$, $\log 1/\delta$, the size of the domain, and the size of 
$\alpha_1\wedge\cdots\wedge\alpha_k$.
\end{proposition}
\begin{proof}
We will first argue that $W$ has efficient PAC-Certification for $A$. Following
the argument sketched above, let any $\alpha\in A$ and $\epsilon,\delta,\gamma
\in (0,1)$ be given. We then simply run our hypothetical algorithm with query
$\alpha$ and $H$ empty. We know that this algorithm then runs in time polynomial
in $|\alpha|$, $1/\gamma$, $\log 1/\delta$, and $\log 1/\epsilon$. Furthermore,
if $\alpha$ is not $(1-\epsilon-\gamma)$-valid (i.e., $\Pr_{x\in D}[\alpha(x)=0]
\geq\epsilon+\gamma$), then we know the algorithm must detect this with
probability $1-\delta$. Likewise, if $\alpha$ satisfies 
$\Pr_{\rho\in M(D)}[W(\alpha,\rho)]\geq 1-\epsilon+\gamma$, then for $A'=
\{\alpha\}$, there is a proof of $\alpha$ from $A'$ in $\calS$ and our algorithm
is guaranteed to recognize that we are in the second case with probability
$1-\delta$. So we see that the efficient PAC-Certification of $W$ for $A$ is
immediate.

Let any partial assignment $\rho$ be given, and consider the family of point
distributions $D_y$ for $y$ consistent with $\rho$ with the masking process $M$
that obscures precisely the entries hidden in $\rho$. Then for every such $y$,
the distribution $M(D_y)$ is a point distribution that produces $\rho$ with
probability 1. Consider the behavior of the algorithm for efficient 
PAC-Certification of $W$ for $A$ given access to such a distribution
(which is trivially simulated given $\rho$) with say $\epsilon=1/2$,
$\gamma=1/4$.

Suppose that $\rho$ is consistent with some $y$ for which $\alpha(y)=0$. Then
in such a case, $\Pr_{x\in D_y}[\alpha(x)=0]=1\geq\epsilon+\gamma$, so when
given examples from $M(D_y)$ (and hence, when given $\rho$ as every example) the
algorithm must decide that the second case holds. Now, suppose on the other
hand that $W(\alpha,\rho)$ holds; then since our distribution produces
$\rho$ with probability 1, the algorithm must decide the first case holds.
Thus, our modified algorithm is as needed for the first part.

For the second part, we note that running the algorithm from the first part on
each example and each partial assignment from a sample of size $O(1/\gamma^2\log
1/\delta)$, and checking whether the fraction of times $W$ was decided to hold
for all $k$ formulas exceeded $1-\epsilon$ suffices to distinguish the two cases
by the usual concentration bounds.
\end{proof}

% So, in particular, if $W(\alpha_1,\rho)\wedge\cdots\wedge W(\alpha_k,\rho)$
% implies $W(\alpha_1\wedge\cdots\wedge\alpha_k,\rho)$ (as one might expect,
% since $W$ certifies validity of the formula) we must also be able to
% distinguish $W$ on conjunctions of axioms.

Our notion of witnessed values is clearly one that suffices for any family of
axioms $A$. By contrast, we now see that for example, we cannot in general
take $W$ to be the collection of pairs $(\alpha,\rho)$ such that for every $x$
consistent with $\rho$ $\alpha(x)=1$ -- arguably, the most natural candidate
(and in particular, the notion originally used by Michael~\cite{michael10}) --
since this may be NP-complete, e.g., for 3-DNF formulas, and so is presumably
not feasible to check. (We remark that our notion actually coincides with this
one in the case of {\em CNF} formulas, which is the relevant class of formulas
for the resolution proof system.)

\section{On the analysis of the algorithm for bounded-space treelike resolution}
\label{space-bound-alg-appendix}

We note that we can associate an optimal clause space to a given derivation
using the following recurrence (often used to define the equivalent {\em pebble
number} of a tree):

\begin{proposition}\label{optimal-space}
The optimal space derivation for a treelike resolution proof corresponding to a
given tree can be obtained recursively as follows:
\begin{compactitem}
\item The space of a single node is $1$.
\item The space of the root of a tree with two subtrees derivable in space $s$
is $s+1$.
\item The space of the root of a tree with subtrees derivable in space $s>s'$ is
$s$.
\end{compactitem}
\end{proposition}
\begin{proof}
We proceed by induction on the structure of the tree, of course, and a proof
of a clause must assert that clause in the final step, so any proof must use
one clause's worth of space (which is attained for the sources -- axioms -- of
the proof). Furthermore, it is clear that for any node of a tree, given that
the formula holds for the subtrees rooted at that node, the formula continues
to hold: if one subtree requires more space than the other, we can derive the
clause labeling the root of the former tree in space $s$, and retaining that
clause on the backboard, we can carry out the space $s'$ derivation for the
other subtree on the blackboard utilizing total space $s'+1\leq s$. This
derivation is optimal since the proof derives the clauses labeling the roots of
both subtrees, and therefore it requires at least as much space as the
derivation of either subtree.

If the subtrees both require space $s$, then using a derivation similar to the
one described above (for the subtrees in arbitrary order) gives a space $s+1$
derivation of the root. To see that this is optimal, we first note that if the
blackboard is ever empty during a resolution proof, we could eliminate any
steps prior to the step with the empty blackboard, and still obtain a legal
proof, so we assume WLOG that the derivation when restricted to either of the
subtrees always include at least one clause. We next note that in any derivation
of one of the subtrees, by the induction hypothesis, there must be some
blackboard configuration that contains $s$ clauses. If this occurs during a
derivation of the other subtree in the overall derivation, then the overall
derivation uses at least $s+1$ space. If it does not, then the conclusion of
this derivation (the root of the subtree) must remain on the blackboard for
use in the final step of the proof; therefore, at a configuration of the
blackboard in the derivation of the other subtree with at least $s$ clauses, at
least $s+1$ clauses appear on the blackboard in the overall derivation.
\end{proof}

Actually, Ans\'otegui et al.~\cite{ablm08} refer to the clause space for
treelike resolution as the {\em Horton-Strahler number} after the discoverers of
the corresponding combinatorial parameter on trees~\cite{horton45,strahler52}
(which again happens to be essentially the same as the ``pebble number'' of
the tree). The algorithm for efficient proof search -- SearchSpace,
Algorithm~\ref{search-space} -- was, to the best of our knowledge, first
essentially discovered as an algorithm for learning decision trees (of low
pebble number) by Ehrenfeucht and Haussler~\cite{eh89}, (we remark that the
connection between treelike resolution and decision trees is an old bit of
folklore, first appearing in the literature in a work by Lov\'asz et al.~\cite
{lnnw95}) and rediscovered in the context of resolution by Kullmann~\cite
{kullmann99}; the algorithm used by Beame and Pitassi~\cite{bp96} is also
essentially similar, although they only considered the resulting proof tree size
(not its space).

Although the analysis of SearchSpace is, at its heart, a fairly straightforward
recurrence, it requires some groundwork. We first note that whenever a bounded
space treelike resolution proof exists, it can be converted into a (normal) form
that can be discovered by SearchSpace:

\begin{definition}[Normal]
We will say that a resolution proof is {\em normal} if in its corresponding DAG:
\begin{inparaenum}
\item All outgoing edges from Cut nodes are directed to Cut nodes.
\item The clauses labeling any path from the sink to a Cut node contain literals
using every variable along the path.
\item A given variable is used in at most one cut step and at most one weakening
step along every path from a source to a Cut node.
\end{inparaenum}
\end{definition}

\begin{proposition}\label{normal-suffices}
For any space-$s$ treelike resolution proof $\Pi$ there is a normal
space-$s$ treelike resolution proof $\Pi'$.
\end{proposition}
\begin{proof}
First note that in general, we don't need to use weakening steps in the proof,
except perhaps on some initial path from a source: all other occurrences
can be eliminated by deleting the introduced literal along the path to the
sink until either a node is encountered in which the other incoming edge is from
a clause that also features that literal or which applies the cut rule on that
variable, redirecting the edge on this path to the cut node past it towards the
sink in the latter case (eliminating the other branch of the proof), and then 
finally replacing the weakening node with the node leading to it. This 
transformation does not increase the clause space of a proof and leaves a 
treelike proof treelike.

Once the weakening steps have been removed (i.e., in the proof cut nodes only
have outgoing edges to other cut nodes) we can see that on any path from the
sink to any cut node, at most one literal is introduced at each step; in
particular, the set of literals on the path leading to any cut node is a 
superset of the literals in the cut node. Note that we can obtain a proof of 
the same clause space in which the internal nodes are all labeled with the 
clauses consisting of these sets of literals, by adding some additional 
weakening steps between the sources of the proof and the first cut node. Since 
these steps leave these chains at clause space 1, the clause space is 
preserved, and a treelike proof is still treelike.

Finally, to guarantee the third property, we show how to eliminate additional
mentions of a variable. While the proof is not normal, identify some offending 
path. For the subtree rooted at the occurrence of the label closest to the 
source of this path, replace this subtree with its child subtree labeled with 
the same clause (note that one such subtree must exist since this literal is 
already mentioned in the clause). Note that the result is still a treelike 
resolution proof, and moreover, since the child subtree has clause space no 
greater than the clause space of the original subtree, the clause space of
the new proof cannot increase. 
\end{proof}

We now describe the proof of Theorem~\ref{space-analysis}.

\begin{theorem}[SearchSpace finds space-$s$ treelike proofs when they exist]
If there is a space-$s$ treelike proof of a clause $C$ from a CNF formula
$\varphi$, then SearchSpace returns such a proof, and otherwise it returns
``none.'' In either case, it runs in time $O(|\varphi|\cdot n^{2(s-1)})$ where
$n$ is the number of variables.
\end{theorem}
\begin{proof}
Recalling Proposition~\ref{optimal-space}, in any normal space-$s$ treelike
derivation of a clause $C$, one of the clauses involved in the final step must
be derivable in space at most $s-1$. It therefore clear that SearchSpace can
find any normal space-$s$ treelike proof by tracing paths from the root,
choosing a literal labeling one of the clauses derivable in strictly smaller
space first. By Proposition~\ref{normal-suffices}, this is sufficient, and all
that remains is to check the running time.

Given $W$ work per each invocation of SearchSpace (i.e., ignoring its recursive
calls, so $T(n,1)\leq W$ for all $n$ and $T(1,s)\leq W$ for all $s$), the
running time is described by the recurrence $T(n,s)\leq T(n-1,s)+2nT(n-1,s-1)+
W$. We can verify (by induction on $n$ and $s$) that $W(n+1)^{2(s-1)}$ is a
solution. Assuming the bound holds for $T(n-1,s)$ and $T(n-1,s-1)$, (for $n>1$,
$s>1$):
\begin{align*}
Wn^{2(s-1)}+2n\cdot W\cdot n^{2(s-2)}+W
&= W((n+2)\cdot n^{2s-3}+1)\\
&\leq W((n+1)^{2(s-1)}\frac{(n+2)n}{(n+1)^2}+1)\\
&\leq W((n+1)^{2(s-1)}-\frac{1}{(n+1)^2}(n+1)^{2(s-1)}+1)\\
&\leq W(n+1)^{2(s-1)}
\end{align*}
Noting that the first case can be checked in time $O(|\varphi|)$ (for
$O(|\varphi|)$ work per node) gives the claimed bound.
\end{proof}

We now establish that the bounded-space algorithm efficiently finds treelike
proofs; we first recall the statement of Proposition~\ref{treelike-log-space}.

\begin{proposition}
A treelike proof $\Pi$ can be carried out in clause space at most $\log_2
|\Pi|+1$.
\end{proposition}
\begin{proof}
We proceed by induction on the structure of the DAG corresponding to $\Pi$.
For a proof consisting of a single node, the claim is trivial. Consider any
treelike proof now; one of children of the root is the root of a subtree
containing at most half of the nodes of the tree. By the induction hypothesis,
this derivation can be carried out in space at most $\log_2(|\Pi|/2)+1=
\log_2|\Pi|$, while the other child can be derived in space at most
$\log_2|\Pi|+1$. Therefore, by Proposition~\ref{optimal-space}, there is a
derivation of the root in space at most $\log_2|\Pi|+1$.
\end{proof}

\bibliographystyle{plain}

\bibliography{robust}

\end{document}